\documentclass{article}
\usepackage[margin=1in]{geometry}
\PassOptionsToPackage{numbers, compress}{natbib}

\usepackage{natbib}
\usepackage{microtype}
\usepackage{graphicx}
\usepackage{booktabs} 

\usepackage{hyperref}

\usepackage{thmtools,thm-restate}
\usepackage{amsmath}
\usepackage{amssymb}
\usepackage{mathtools}
\usepackage{amsthm}
\newcommand{\floor}[1]{\left\lfloor #1 \right\rfloor}
\usepackage[capitalize,noabbrev]{cleveref}

\theoremstyle{plain}
\newtheorem{theorem}{Theorem}[section]

\newtheorem{lemma}[theorem]{Lemma}
\newtheorem{corollary}[theorem]{Corollary}
\theoremstyle{definition}
\newtheorem{definition}[theorem]{Definition}

\theoremstyle{remark}
\newtheorem{remark}[theorem]{Remark}
\newcommand{\State}{\STATE}
\newcommand{\For}{\FOR}
\newcommand{\EndFor}{\ENDFOR}

\makeatletter
\newtheorem*{rep@theorem}{\rep@title}
\newcommand{\newreptheorem}[2]{%
	\newenvironment{rep#1}[1]{%
		\def\rep@title{#2 \ref{##1}}%
		\begin{rep@theorem}}%
		{\end{rep@theorem}}}
\newcommand{\mcf}{\mathcal}

\usepackage{amsmath}
\usepackage{graphicx}
\usepackage{algorithm, algorithmic}
\usepackage{amsmath}
\usepackage{amssymb}
\usepackage{subcaption}
\usepackage{xcolor}
\usepackage{amsthm}

\usepackage[capitalize,noabbrev]{cleveref}

\newcommand{\innerprod}[2]{\left\langle{#1},{#2}\right\rangle}
\DeclareMathOperator*{\argmin}{arg\,min}
\DeclareMathOperator*{\argmax}{arg\,max}
\newcommand{\expert}{{\pi_{\textup{E}}}}

\newcommand{\mbs}{\boldsymbol}
\newcommand{\cost}{c}
\newcommand{\true}{c_{\textup{true}}}

\newcommand{\initial}{\mbs{\nu}_0}
\usepackage{dsfont}

\newcommand{\sspace}{\mcf{S}}
\newcommand{\aspace}{\mcf{A}}

 \newcommand{\norm}[1]{\left\| {#1} \right\|}

\newcommand{\bc}[1]{\left\{{#1}\right\}}
\newcommand{\br}[1]{\left({#1}\right)}
\newcommand{\bs}[1]{\left[{#1}\right]}
\newcommand{\abs}[1]{\left| {#1} \right|}
\usepackage[textsize=tiny]{todonotes}
\usepackage{subcaption}
\newcommand\blfootnote[1]{%
  \begingroup
  \renewcommand\thefootnote{}\footnote{#1}%
  \addtocounter{footnote}{-1}%
  \endgroup
}
\title{ IL-SOAR : Imitation Learning with Soft Optimistic  Actor cRitic}

\author{
\and
\textbf{Stefano Viel*\blfootnote{Equal contribution.}} \\
\texttt{stefano.viel@epfl.ch} \\
EPFL 
\and
\textbf{Luca Viano*} \\
\texttt{luca.viano@epfl.ch} \\
LIONS \\
EPFL
\and
\textbf{Volkan Cevher} \\
\texttt{volkan.cevher@epfl.ch} \\
LIONS \\
EPFL
}

\begin{document}

\maketitle



\begin{abstract}
\noindent This paper introduces the SOAR framework for imitation learning. SOAR is an algorithmic template 
that learns a policy from expert demonstrations with a primal dual style algorithm that alternates cost and policy updates. Within the policy updates, the SOAR framework uses an actor critic method with multiple critics to estimate the critic uncertainty and build an optimistic critic fundamental to drive exploration.

\noindent When instantiated in the tabular setting, we get a provable algorithm with guarantees that matches the best known results in the desired accuracy parameter $\epsilon$.

\noindent Practically, the SOAR template can boost the performance of \emph{any} imitation learning algorithm based on Soft Actror Critic (SAC).  As an example, we show that SOAR can boost consistently the performance of the following SAC-based imitation learning algorithms: $f$-IRL, ML-IRL and CSIL. Overall, thanks to SOAR, the required number of episodes to achieve the same performance is reduced by half.\footnote{Project code available at \url{https://github.com/stefanoviel/SOAR-IL/tree/master}}
\end{abstract}

\section{Introduction}
Several recent state of the art imitation learning (IL) algorithms \cite{ni2021f,zeng2022maximum,Garg:2021,watson2023coherent,viano2022proximal} are built on Soft Actor Critic (SAC) \cite{Haarnoja:2018} to perform the policy updates.
SAC uses \emph{entropy} regularized policy updates to maintain a strictly positive probability of taking each action. However, this is known to be an inefficient exploration strategy if deployed alone \cite{cesa2017boltzmann}.

Indeed, several recent theoretical imitation learning achieve performance guarantees by adding exploration bonuses on top of the regularized policy updates, which encourage the learner to visit state-action pairs that have not been visited previously. Unfortunately, such works are only available in the tabular setting \cite{Shani:2021,xu2023provably} and in the linear setting \cite{viano2024imitation}. The design of the exploration bonuses in these works is strictly tight to the tabular or linear structure of the transition dynamics, therefore, these analyses offer little insight on how to design an efficient exploration mechanism using neural network function approximation.

There is, therefore, a lack of a technique that satisfies the following two requirements.
\begin{itemize}
\vspace{-2mm}
\item It is statistically and computationally efficient in the tabular setting.
\vspace{-2mm}
\item It can be implemented easily in continuous states and actions problems requiring neural networks function approximation.
\end{itemize}
In this paper, we present a general template, dubbed Soft Optimistic Actor cRitic Imitation Learning (SOAR-IL) satisfying these requirements. 

The main idea is to act according to an \emph{optimistic} critic within the SAC block on which many IL algorithms rely.
Here, optimism means appropriately underestimating the expected cumulative cost incurred by playing a policy in the environment. This principle known as \emph{optimism in the face of uncertainty} has led to several successful algorithms in the bandits community.
\begin{figure*}[t]
    \centering
\includegraphics[width=\textwidth]{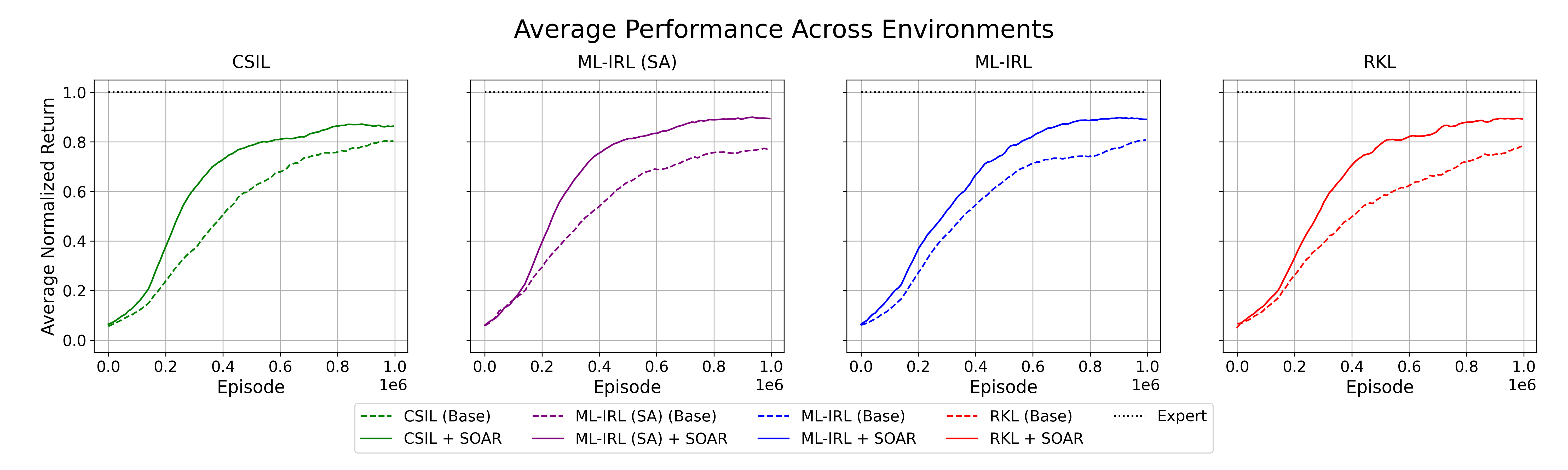}
\vspace{-4mm}
    \caption{\small{Summary of experimental results. Each plot compares the average normalized return across $4$ MuJoCo environments with $16$ expert trajectories for a base algorithm and its SOAR-enhanced version. SOAR replaces the single critic in SAC-based methods with multiple critics to compute an optimistic estimate. Across all algorithms, incorporating SOAR consistently improves performance. ML-IRL (SA) stands for ML-IRL \cite{zeng2022maximum} from expert state-action demonstrations.}}
    \label{fig:average}
\end{figure*}

While optimism is often achieved using the structure of the problem (tabular, linear, etc.), in this work, we build optimistic estimators using an ensemble technique. That is, multiple estimators for the same quantity are maintained and aggregated to obtain an optimistic estimator.
This technique scales well with deep imitation learning.
To summarize, we have the following contributions.
\paragraph{Theoretical contribution}
We show that there exists a computationally efficient algorithm that uses an ensemble based exploration technique that gives access to $\mathcal{O}(\epsilon^{-2})$ expert trajectories and $\mathcal{O}(\epsilon^{-2})$ interactions in a tabular MDP outputs a policy such that its cumulative expected cost is at most $\epsilon$ higher than the expert cumulative expected cost with high probability.
\paragraph{Practical Contribution} We apply an ensemble-based exploration technique, SOAR, to boost the performance of deep imitation learning algorithms built on SAC, demonstrating its effectiveness on MuJoCo environments. Specifically, we show that incorporating SOAR consistently boosts the performance of base methods such as Coherent Soft Imitation Learning (CSIL)\cite{watson2023coherent}, Maximum Likelihood IRL (ML-IRL)\cite{zeng2022maximum} and RKL \cite{ni2021f}. As shown in Figure~\ref{fig:average}, our approach consistently outperforms the base algorithms across all MuJoCo environments. Notably, SOAR achieves the best performance of the baselines requiring only approximately half the number of learning episodes.

\section{Preliminaries and Notation}
The environment is abstracted as Markov Decision Process (MDP) \cite{Puterman:1994} 
which consists of a tuple $(\sspace, \aspace, P, c, \initial, \gamma)$ where $\sspace$ is the state space, $\aspace$ is the action space, $P: \sspace\times\aspace \rightarrow \Delta_{\sspace}$ is the transition kernel, that is, 
$P(s'|s,a)$ denotes the probability of landing in state $s'$ after choosing action $a$ in state $s$. Moreover, $\initial$ is a distribution over states from which the initial state is sampled. Finally, $c: \sspace\times \aspace \rightarrow [0,1]$ is the cost function, and $\gamma \in [0,1)$ is called the discount factor.

\textbf{Value functions and occupancy measures} We define the state value function at state $s \in \sspace$ for the policy $\pi$ under the cost function $c$ as $V_c^{\pi}(s) \triangleq \mathbb{E}\bs{\sum^{\infty}_{h=0} \gamma^{h} c(s_h,a_h) | s_1 = s}$.
The
expectation over both the randomness of the transition dynamics and the one of the learner's policy.
Another convenient quantity is the occupancy measure of a policy $\pi$ denoted as $d^{\pi}\in \Delta_{\sspace\times\aspace}$ and defined as follows
$d^{\pi}(s,a) \triangleq (1 - \gamma)\sum^{\infty}_{h=0} \gamma^{h}\mathbb{P}\bs{s,a \text{ is visited after $h$ steps acting with $\pi$}}$. We can also define the state occupancy measure
as $d^{\pi}(s) \triangleq (1 - \gamma)\sum^{\infty}_{h=0} \gamma^{h}\mathbb{P}\bs{s\text{ is visited after $h$ steps acting with $\pi$}}$.

\textbf{Imitation Learning} In imitation learning, the learner is given a dataset $\mathcal{D}_{\expert}$ of expert trajectories collected by an unknown expert policy $\expert$.\footnote{In order, to accommodate state-only and state-action with a unified analysis we overload the notation for the expert dataset. $\mathcal{D}_\expert$ denotes a collection of samples from the expert state occupancy measure in the former case and a collection of state-actions sampled from the state-action occupancy measure in the latter case.}
By trajectory $\boldsymbol{\tau}^k$, we mean the sequence of states and actions sampled rolling out the policy $\pi_k$ for a number of steps sampled from the distribution $\mathrm{Geometric}(1 - \gamma)$.
 Given $\mathcal{D}_{\expert}$, the learner adopts an algorithm $\mathcal{A}$ to learn a policy $\pi^{\mathrm{out}}$ such that $\innerprod{\initial}{V^{\widehat{\pi}^k}_{c_{\mathrm{true}}}  - V^{\expert}_{c_{\mathrm{true}}} } \leq \epsilon$ with high probability.
 
We use the notation $\mathbf{e}_s$ to denote a vector in $\mathbb{R}^{\abs{\sspace}}$ zero everywhere but in the coordinate corresponding to the state $s$ ( for an arbitrary ordering of the states). Analogously, we use  $\mathbf{e}_{s,a}$ to denote a vector in $\mathbb{R}^{\abs{\sspace}\abs{\aspace}}$ zero everywhere but in the $(s,a)^{th}$ entry which equals one.
\section{The Algorithm}
\begin{algorithm}
\caption{SOAR-Imitation Learning \label{alg:meta}}
\begin{algorithmic}[1]
\REQUIRE Reward step size $\alpha$, Expert dataset $\mathcal{D}_\expert$, Discount factor $\gamma$, Policy step size $\eta$.
\State Initialize $\pi^1$ as uniform distribution over $\mathcal{A}$.
\State Initialize empty replay buffer,i.e. $\mathcal{D}^0 = \bc{}$
\For{$k = 1$ to $K$}
    \State $\tau^k \gets \textsc{CollectTrajectory}(\pi^{k})$
    \State Add $\tau^k$ to replay buffer,i.e. $\mathcal{D}^{k} = \mathcal{D}^{k - 1} \cup \tau^k $.
    \State $ c^k \gets \textsc{UpdateCost}(c^{k-1}, \mathcal{D}_\expert, \mathcal{D}^k, \alpha)$
    \For{$\ell = 1$ to $L$}
        \State Compute estimator $Q^k_\ell$.
    \EndFor
    \State $Q^k = \textsc{OptimisticQ}( \bc{Q^k_\ell}^L_{\ell=1} )$.
    \State $\pi^k(a|s) = \textsc{PolicyUpdate}(\eta, \bc{Q^\tau(s,a)}^k_{\tau=1})$ 
\EndFor
\end{algorithmic}
\end{algorithm}
In this Section, we describe \Cref{alg:meta}. A meta-algorithm that encompasses several existing imitation learning algorithms.
Inside each iteration of the main for loop, the learner collects a new trajectory sampling actions from the policy $\pi^k$ (Line 4 in \Cref{alg:meta}) and then performs the following steps.
\begin{itemize}
    \item \textbf{The Cost update.} At Line 6 of \Cref{alg:meta}, the learner updates an estimate of the true unknown cost function with the algorithm-dependent routine $\textsc{UpdateCost}$. 
    For instance, Generative Adversarial Imitation Learning (GAIL), Adversarial Inverse Reinforcement Learning (AIRL), and Discriminator Actor Critic (DAC) \cite{Ho:2016b,Fu:2018,Kostrikov:2019} use a reward derived from a discriminator neural network trained to distinguish state-action pairs visited by the expert from those visited by the learner. 
    
    Using a fixed cost function obtained from a behavioral cloning warm up is, instead, the approach taken in CSIL \cite{watson2023coherent}. Moreover, updating the reward to minimize an information theoretic divergence between expert and learner state occupancy measure is the approach taken in RKL \cite{ni2021f}. Finally, \cite{zeng2022maximum} updates the cost using online gradient descent (OGD) \cite{Zinkevich2003}.
    \item \textbf{The state-action value function update.} In the \emph{for} loop at Lines 7-9 of \Cref{alg:meta}, the learner updates $L$ different critics trained on different subsets of the data sampled from the replay buffer $\mathcal{D}^k$ denoted as $\bc{\mathcal{D}^k_\ell}^L_{\ell=1}$. For a fixed state-action pair each dataset contains independent samples from $P(\cdot|s,a)$. This allows creating $L$ jointly independent random variables $\bc{Q^k_\ell}^L_{\ell=1}$ that estimate the ideal value iteration update (i.e. $c^k + \gamma P V^k$) which cannot be implemented exactly due to the lack of knowledge on $P$. 
    
    In \Cref{sec:tabular}, we provide an explicit way to compute $L$ slightly optimistic estimates for the tabular setting. Moreover, in the deep imitation learning experiments, we train $L$ different critics via temporal difference, as it is commonly done in Soft Actor Critic implementation ( see \cite{Haarnoja:2018} ). 
    
    Finally, in Line 10 of \Cref{alg:meta}, the $L$ critics are aggregated to generate an estimate $Q^{k+1}$ which is, with high probability, optimistic i.e. $Q^{k+1} \leq c^k + \gamma P V^k$. In other words, it underestimates the update that could have been performed by value iteration if the transition matrix $P$ was known to the learner.
    We provide aggregation routines that satisfy this requirement if $L$ is large enough.
    \item \textbf{The policy update} As the last step of each inner loop, the learner updates the policy using the optimistic state-action value function estimate.
    In the tabular case, we will instantiate the update using an online mirror descent (OMD) step \cite{Beck:2003,nemirovskij1983problem} (also known as the multiplicative weights update \cite{warmuth1997continuous,auer1995gambling}).
    As it will be evident from \Cref{sec:analysis}, this update we can ensure that the KL divergence between consecutive policies is upper bounded in terms of the policy step size $\eta$.
    For the continuous state-action experiments, the online mirror descent is approximated via a gradient descent step on the SAC loss.
\end{itemize} 
\begin{remark}
Notice that only one pair of critics is used in the implementation of SAC ($L =1$) that serves as base RL algorithm for several commonly used IL algorithms ( GAIL \cite{Ho:2016b}, AIRL \cite{Fu:2018}, IQ-Learn \cite{Garg:2021}, PPIL \cite{viano2022proximal}, RKL \cite{ni2021f} and ML-IRL \cite{zeng2022maximum}). As proven in \Cref{cor:optimism}, $L=1$ is not enough to ensure optimism, not even in the tabular case.
In our experiments, we show that a value of $L$ larger than $1$ is beneficial in all the MuJoCo environments we tested on.
\end{remark}

\subsection{Algorithm with guarantees in the tabular case}
\label{sec:tabular}

We consider an instance of \Cref{alg:meta} in the tabular case for which we will prove theoretical sample efficiency guarantees. We present the pseudocode in \Cref{alg:theory_version}.

For what concerns the analysis, the first step is to extract the policy achieving the sample complexity guarantees above via an online-to-batch conversion. 
That is, the output policy is sampled uniformly from a collection of $K$ policies $\bc{\pi^k}^K_{k=1}$. 
The sample complexity result follows from proving that the policies $\bc{\pi^k}^K_{k=1}$ produced by \Cref{alg:theory_version} is a sequence with sublinear regret in high probability. More formally, we define the regret as follows.
\begin{definition}\label{def:regret}
\textbf{Regret}
The regret is defined as follows
\begin{equation*}
\mathrm{Regret}(K)\triangleq\frac{1}{1 - \gamma}\sum^K_{k=1} \innerprod{\true}{d^{\pi^k} - d^{\expert}}
\end{equation*}
\end{definition}
\begin{remark}
Notice that the regret defined in this way satisfies $\mathrm{Regret}(K) = \sum^K_{k=1} \innerprod{\initial}{V^{\pi^k}_{c_{\mathrm{true}}} - V^{\expert}_{c_{\mathrm{true}}} }$. For this reason, we require the factor $(1-\gamma)^{-1}$ in the definition.
\end{remark}
Omitting dependencies on the horizon and the state action spaces cardinality, we will guarantee that
$$
\mathrm{Regret}(K) \leq \mathcal{O}(K^{1/2} + K \abs{\mathcal{D}_\expert}^{-1/2}),
$$
with high probability.
Notice that this bound is sublinear in $K$, for  $\abs{\mathcal{D}_\expert} = \mathcal{O}(K)$.
To obtain such bound, we adopt the following decomposition for $(1-\gamma)\mathrm{Regret}(K)$ adapted from \cite{Shani:2021} to accommodate the infinite horizon setting.
\begin{equation}
    \underbrace{\sum^K_{k=1} \innerprod{c^k}{d^{\pi^k} - d^{\expert}}}_{:= (1-\gamma)\mathrm{Regret}_{\pi}(K,\expert)} + \underbrace{\sum^K_{k=1} \innerprod{c_{\mathrm{true}} - c^k}{d^{\pi^k} - d^{\expert}}}_{:= (1-\gamma)\mathrm{Regret}_c(K, c_{\mathrm{true}})} \label{eq:dec}
\end{equation}
\begin{algorithm}
\caption{Tabular SOAR-IL \label{alg:theory_version}}
\begin{algorithmic}[1]
\REQUIRE Step size $\eta$, Expert dataset $\mathcal{D}_\expert$, Discount factor $\gamma$, Reward step size $\alpha$, $N^0(s,a)=0$ for all $s,a$, number of estimators $L= 36 \log \br{\abs{\sspace}\abs{\aspace} K /\delta}$.
\State Initialize $\pi^1$ as uniform distribution over $\mathcal{A}$

\For{$k = 1$ to $K$}
    \State Sample trajectory length $L^k \sim \mathrm{Geometric}(1 - \gamma)$.
    \State $\tau^k= \bc{(s^k_t, a^k_t)}^{L^k}_{t=1}$ rolling out $\pi^k$ for $L^k$ steps.
    \State Update counts for all $s^k_t, a^k_t \in \tau_k$:
    $$N^k(s^k_t, a^k_t) = N^{k-1}(s^k_t, a^k_t) + 1.$$
    \State Add $s^k_t, a^k_t, s^k_{t+1}$ to the datasets with index $\ell = N^k(s^k_t, a^k_t)  \mod L$, $$\mathcal{D}_\ell^k = \mathcal{D}_\ell^{k-1} \cup \bc{s^k_t, a^k_t},~~~~ \mathcal{R}_\ell^k = \mathcal{R}_\ell^{k-1} \cup \bc{s^k_t, a^k_t,s^k_{t+1}}. $$ 
    \State $ c^k = \textsc{CostUpdateTabular}(c^{k-1}, \tau^k, \mathcal{D}_\expert).$
    \For{$\ell = 1$ to $L$}
        \State $N^k_\ell(s,a,s') = \sum_{\bar{s},\bar{a},\bar{s}'\in \mathcal{R}^k_\ell} \mathds{1}_{\bc{\bar{s},\bar{a},\bar{s}' = s,a,s'}}.$
        \State $N^k_\ell(s,a) = \sum_{\bar{s},\bar{a}\in \mathcal{D}^k_\ell} \mathds{1}_{\bc{\bar{s},\bar{a} = s,a}}.$
        \State $\widehat{P}^k_\ell(\cdot|s,a) = \frac{N^k_\ell(s,a,\cdot)}{N^k_\ell(s,a) + 2}$
    \EndFor
    \State $Q^{k+1} = \textsc{OptimisticQTabular}(V^k, \bc{\widehat{P}^k_\ell}^L_{\ell=1}, c^k)$.
    \State $\pi^{k+1}(a|s) \propto \pi^{k}(a|s) \exp\br{-\eta Q^{k+1}(s,a)}$
    \State $ V^{k+1}(s) = \innerprod{\pi^{k+1}(\cdot|s)}{Q^{k+1}(s,\cdot)}$
\EndFor
\State \textbf{Return} The mixture policy $\widehat{\pi}^K$.
\end{algorithmic}
\end{algorithm}

The algorithmic design for the tabular setting aims at updating the cost variable so that the term $\mathrm{Regret}_c$ grows sublinearly (see Line 7 in \Cref{alg:theory_version}). 

We consider both cases of imitation from state-action expert data (Lines 4-6 of $\textsc{CostUpdateTabular}$ ) and state-only expert data (Lines 2-3 of $\textsc{CostUpdateTabular}$). These cases differ only in the stochastic loss for the cost update. Notice that we overload the notation to address both state-only and state-action imitation learning with a unified analysis.
In particular, $\widehat{d^{\pi^k}}$ is an unbiased estimate of the learner occupancy measure. For state-only imitation learning we use $ \widehat{d^{\pi^k}} = \mathbf{e}_{s^k_{L^k}}$ and estimated expert occupancy measure equals to $\widehat{d^\expert} = \abs{\mathcal{D}_\expert}^{-1}\sum_{s\in \mathcal{D}_\expert} \mathbf{e}_s$ while for state-action imitation learning $ \widehat{d^{\pi^k}} = \mathbf{e}_{s^k_{L^k}, a^k_{L^k}} $ and $\widehat{d^\expert} = \abs{\mathcal{D}_\expert}^{-1}\sum_{s,a\in \mathcal{D}_\expert} \mathbf{e}_{s,a}$.
The formal bound on $\mathrm{Regret}_c$ is given in \Cref{thm:reward_regret_bound}.
\begin{algorithm}
\caption{\textsc{CostUpdateTabular}}
\begin{algorithmic}[1]
\REQUIRE Current cost vector $c^{k-1}$, trajectory $\tau^k$, expert dataset $\mathcal{D}_\expert$.
\IF {$\textsc{State-Only} = \textsc{TRUE}$}
    \State $\widehat{d^{\pi^k}} = \mathbf{e}_{s^k_{L^k}}$.
    
    \State $\widehat{d^\expert} = \abs{\mathcal{D}_\expert}^{-1}\sum_{s\in \mathcal{D}_\expert} \mathbf{e}_s$
    \ELSE
    \State $\widehat{d^{\pi^k}} = \mathbf{e}_{s^k_{L^k},a^k_{L^k}}$.
    
    \State $\widehat{d^\expert} = \abs{\mathcal{D}_\expert}^{-1}\sum_{s,a \in \mathcal{D}_\expert} \mathbf{e}_{s,a}$
    \ENDIF
    \State \textbf{Return:} $ c^{k} \gets \Pi_{\mathcal{C}}\bs{c^{k-1} - \alpha (\widehat{d^{\expert}} - \widehat{d^{\pi^k}} )} $
\end{algorithmic}
\end{algorithm}
The rest of the algorithm aims to provide a sublinear bound on $\mathrm{Regret}_\pi$.
In particular, the updates for the estimated transition kernels $\bc{\widehat{P}^k_\ell}^L_{\ell=1}$ in Lines 8-12 of \Cref{alg:theory_version} serves to build $L$ slightly optimistic \footnote{The optimism is achieved by adding $2$ in the denominator of the estimated transition kernels.} estimate of the ideal value function update. 

In the routine $\textsc{OptimisticQTabular}$, we propose two aggregation rules to generate the optimistic $Q$ value estimate to be used in the policy update step. The first one, takes the minimum of the $L$ estimators as in \Cref{eq:update1}, while the second option \Cref{eq:update2} considers the mean of the $L$ estimators minus a factor proportional to the empirical standard deviation. By Samuelson's inequality \cite{samuelson1968deviant}, we prove that the second option is more optimistic.
\begin{algorithm}
\caption{\textsc{OptimisticQTabular}}
\begin{algorithmic}[1]
\REQUIRE current state value function estimate $V^k$, ensemble of estimated transitions $\bc{\widehat{P}^k_\ell}^L_{\ell=1}$, cost $c^k$.
    \State \textcolor{blue}{// Option 1} \begin{equation} \text{\textbf{return}}~~~Q^{k+1} = c^k + \gamma \min_{\ell \in [L]} \widehat{P}^k_\ell V^k \tag{Min}\label{eq:update1}\end{equation}
    \State \textcolor{blue}{// Option 2}
    \begin{equation}
        \text{\textbf{return}}~~~ Q^{k+1} = c^k + \gamma \max\bs{\frac{1}{L}\sum^L_{\ell =1} \widehat{P}^k_\ell V^k - \sigma^k   , 0}\tag{Mean-Std}\label{eq:update2} 
    \end{equation}
    \State with $\sigma^k = \sqrt{\sum^L_{\ell=1} \br{
    \widehat{P}^k_\ell V^k
    - \frac{1}{L} \sum^L_{\ell'=1} \widehat{P}^k_{\ell'} V^k}^2}.$
\end{algorithmic}
\end{algorithm}

Finally, an iteration of the tabular case algorithm is concluded by the policy update implemented via OMD.

Having described our main techniques we are in the position of stating our main theoretical results hereafter.
\begin{theorem}
\label{thm:main_result}\textbf{Main Result}
For any MDP, let us consider either the update \Cref{eq:update1} or \Cref{eq:update2}, it holds that with probability $1-5\delta$ that $\frac{\mathrm{Regret}(K)}{K}$ of Tabular SOAR-IL (\Cref{alg:theory_version}) is upper bounded by
\begin{equation*}
\widetilde{\mathcal{O}}\br{\sqrt{\frac{\abs{\sspace}^4 \abs{\aspace}  \log (1 / \delta)}{(1-\gamma)^5 K}}} + \sqrt{ \frac{\abs{\sspace}^2\abs{\aspace} \log \br{\abs{\sspace}\abs{\aspace}/\delta} (\log(\abs{\sspace}) + 2)^2}{(1-\gamma)^2\abs{\mathcal{D}_\expert}}}.
\end{equation*}
Therefore, choosing $K = \widetilde{\mathcal{O}}\br{\frac{\abs{\sspace}^4 \abs{\aspace}  \log (1 / \delta)}{(1-\gamma)^5 \epsilon^2}}$ and $\abs{\mathcal{D}_\expert} = \frac{\abs{\sspace}^2\abs{\aspace}  \log \br{\abs{\sspace}\abs{\aspace} /\delta} (\log(\abs{\sspace}) + 2)^2}{ \epsilon^2 (1-\gamma)^2}$ it holds that the mixture policy $\widehat{\pi}_K$ satisfies $\innerprod{\initial}{V^{\widehat{\pi}^k}_{c_{\mathrm{true}}}  - V^{\expert}_{c_{\mathrm{true}}} } \leq \epsilon$ with probability at least $1-5 \delta$.
\end{theorem}
\begin{remark}
The bound on $\abs{\mathcal{D}_\expert}$ is the bound on the number of either state-only or state-action expert trajectories depending on the setting considered.
\end{remark}
\begin{remark}
The gurantees are stated for the mixture policy $\widehat{\pi}^K$, i.e. the policy which has an occupancy measure equal to the average occupancy measure of the policies in the no-regret sequence. That is, it holds that $d^{\widehat{\pi}^K} = K^{-1} \sum^K_{k=1} d^{\pi^k}$. The policy $\widehat{\pi}^K$ cannot be computed without knowledge of $P$ but sampling a trajectory from it can be done by choosing an index $k \sim \mathrm{Unif}([K])$ at the beginning of each new episode and continuing rolling out the policy $\pi^k$ for a number of steps sampled from $\mathrm{Geom}(1-\gamma)$.
\end{remark}
\begin{remark}
In the case of state-only expert dataset the provided upper bounds for $K$ and $\abs{\mathcal{D}_\expert}$ are optimal up to log factors in the precision parameters $\epsilon$. Indeed, these upper bounds match the lower bounds in \cite{moulin2025optimistically}.
\end{remark}
\section{Theoretical analysis}
\label{sec:analysis}
We need to start with an important remark on the structure of the MDP considered in the proof.
\begin{remark}
\label{remark:MDP_comment}
For technical reasons, in particular for the proof of \Cref{cor:optimism}, we consider as intermediate step in the proof MDPs where from each state action pairs is possible to observe a transition to only two other possible states.
While this restriction on the dynamics appears to be limiting any MDP can be cast into this form at the cost of a quadratic blow up in the number of states, from $\abs{\sspace}$ to $\abs{\sspace}^2$. To see this, for a general MDP where from a given state action pair a transition to all possible $\abs{\sspace}$ states can be observed 
 is equivalent to a binarized MDP where this \emph{one layer} transition is represented with a tree of depth at most $\log_2 \br{\abs{\sspace}}$ with binary transitions only. Moreover, the discount factor in the binarized MDP should be set to $\gamma_{\mathrm{bin}} = \gamma^{-\log_2\abs{\sspace}}$ to maintain the return unchanged.
 We consider in this section a binarized MDP with $\abs{\sspace}$ states in this section and we squared the number of states in stating \Cref{thm:main_result} which holds for general MDPs. Moreover, in stating the result for general MDP we also inflated the effective horizon by a factor $\log_2\abs{\sspace}$ as shown in \Cref{lemma:eff_horizon}.
 \end{remark}
As mentioned, the proof is decomposed into two main parts: (i) bounding the policy regret $\mathrm{Regret}_\pi$ and (ii) bounding the cost updates regret  $\mathrm{Regret}_c$.
In particular, we can prove the two following results.
\begin{restatable}{theorem}{thmpolicyregret}\label{thm:policy_regret}\textbf{Policy Regret}
In a binarized MDP with $\abs{\sspace}$ states and discount factor $\gamma$, it holds that with probability $1 - 3\delta$, for any policy $\pi^\star$, $\mathrm{Regret}_{\pi}(K,\pi^\star) $ is upper bounded by
\begin{equation*}
\frac{\log \abs{\aspace}}{\eta (1-\gamma)} + \frac{\eta K}{(1-\gamma)^4} + \widetilde{\mathcal{O}}\br{\frac{ \sqrt{ K \abs{\sspace}^2 \abs{\aspace} \log (1/\delta)}}{(1-\gamma)^2}}
\end{equation*}
and for $\eta = \sqrt{ \frac{\log \abs{\aspace} (1-\gamma)^3}{K}}$ it holds that using the update in \eqref{eq:update1} or in \eqref{eq:update2} it holds that $\mathrm{Regret}(K,\pi^\star)$ is upper bounded by $\widetilde{\mathcal{O}}\br{ \sqrt{ \frac{K \abs{\sspace}^2 \abs{\aspace} \log (1/\delta)}{(1-\gamma)^5}}}$.
\end{restatable}
\begin{restatable}{theorem}{thmcostregret}
\label{thm:reward_regret_bound} \textbf{Cost Regret}
In a binarized MDP with $\abs{\sspace}$ states and discount factor $\gamma$, it holds that with probability $1-2\delta$, $(1-\gamma)\mathrm{Regret}_c(K; c_{\mathrm{true}})$ is upper bounded by
\begin{equation*}
    4 \sqrt{K \log (1/\delta)} + K \sqrt{ \frac{\abs{\sspace} \abs{\aspace}\log \br{\abs{\sspace}\abs{\aspace}/\delta}}{2 \abs{\mathcal{D}_\expert}}}
\end{equation*}
\end{restatable}
\begin{remark}
Once \Cref{thm:policy_regret,thm:reward_regret_bound} are proven the bound on \Cref{thm:main_result} follows trivially by a union bound and bounding $\mathrm{Regret}_\pi$ and $\mathrm{Regret}_c$ with \Cref{thm:policy_regret} and \Cref{thm:reward_regret_bound} respectively and dividing everything by $K$ (because in \Cref{thm:main_result} we consider the quantity $\mathrm{Regret}(K)/K$). Finally, we also divide by $1-\gamma$, to match the definition of $\mathrm{Regret}(K)$ in \Cref{def:regret}.
\end{remark}
\subsection{Proof Sketch of \Cref{thm:policy_regret}}
The regret decomposition towards the proof of \Cref{thm:policy_regret} leverages the following Lemma.
\begin{restatable}{lemma}{lemmapdl}
\label{lemma:infinite_extendend_pdl}
Consider the MDP $M = (\sspace, \aspace,  P, \cost, \initial,\gamma)$ and two policies $\pi, \pi^\prime: \sspace \rightarrow \Delta_{\aspace}$. Then consider for any $\widehat{Q} \in \mathbb{R}^{\abs{\sspace}\abs{\aspace}}$ and $\widehat{V}^{\pi}(s) = \innerprod{\pi(\cdot|s)}{\widehat{Q}(s,\cdot)}$ and $Q^{\pi^\prime}, V^{\pi^\prime}$ be respectively the state-action and state value function of the policy $\pi$ in MDP $M$. Then, it holds that $(1-\gamma)\innerprod{\initial}{\widehat{V}^{\pi} - V^{\pi^\prime}}$ equals
\begin{equation*}
      \innerprod{d^{\pi^\prime}}{\widehat{Q} - \cost - \gamma P \widehat{V}^{\pi}} + 
    \mathbb{E}_{s\sim d^{\pi^\prime}}\bs{\innerprod{\widehat{Q}(s,\cdot}{\pi(\cdot|s) -\pi^\prime(\cdot|s)}}.
\end{equation*}
\end{restatable}
\begin{remark}
This Lemma is a generalization of the well-known performance difference Lemma \cite{Kakade:2002} to the case of inexact value functions. Indeed, notice that if $\widehat{Q} = Q^\pi$, then the first term in the decomposition equals zero and the result boils down to the standard performance difference Lemma. For arbitrary $\widehat{Q}$, the first term is a temporal difference error averaged by the occupancy measure $d^{\pi^\prime}$.
\end{remark}
We can apply two times \Cref{lemma:infinite_extendend_pdl} on each of the summands of the sum from $k=1$ to $K$, to obtain a convenient decomposition of $\mathrm{Regret}_\pi$ . Denoting $\delta^{k}(s,a) \triangleq \cost^k(s,a) + \gamma P V^k(s,a) - Q^{k+1}(s,a)$ and $g^k(s,a) \triangleq  Q^{k+1}(s,a) - Q^{k}(s,a)$, we have that
\begingroup
\allowdisplaybreaks
\begin{align}
&(1 - \gamma)  \mathrm{Regret}_\pi(K;\pi^\star) = \nonumber \\ & (1-\gamma) \sum^K_{k=1} \innerprod{\initial}{V^{\pi^k}_{c^k} - V^k + V^k - V^{\pi^\star}_{c^k}} = \nonumber \\ &\sum^K_{k=1} \mathbb{E}_{s\sim d^{\pi^\star}}\bs{
    \innerprod{Q^k(s,\cdot)}{\pi^k(s) - \pi^\star(s)}} \quad \tag{BTRL} \label{eq:main_OMD}\\&\phantom{\leq}+ \sum^K_{k=1}  \sum_{s,a} \bs{d^{\pi^k}(s,a) - d^{\pi^\star}(s,a)}\cdot\bs{\delta^k(s,a)} \tag{Optimism} \label{eq:main_optimism}
    \\&\phantom{\leq}+ \sum^K_{k=1} \mathbb{E}_{s,a\sim d^{\pi^k}}\bs{g^k(s,a)} - \sum^K_{k=1} \mathbb{E}_{s,a\sim d^{\pi^\star}}\bs{g^k(s,a)} \quad \label{eq:main_shift} \tag{Shift}
\end{align}
Next, we bound each of these terms individually.
Starting from the first term, the next Lemma shows that our policy update (Line 14 of \Cref{alg:theory_version}) can be seen as an instance of Be the regularized leader (BTRL) ( see e.g. \cite{orabona2023modern} ). Therefore, it guarantees that for any sequence $\bc{Q^k}^K_{k=1}$, the term \eqref{eq:main_OMD} is bounded as follows.
\begin{restatable}{lemma}{localregret}\label{lemma:local_regret}

Let us consider the sequence of policies $\bc{\pi^k}^K_{k=1}$ generated by \Cref{alg:theory_version}  for all $\eta > 0$ then it holds that $ \eqref{eq:main_OMD} \leq \frac{\log \abs{\aspace}}{\eta}.$
\end{restatable}
Next, we show that thanks to the multiplicative weights update for the policy the KL divergence between consecutive policies is upper bounded by the policy step size $\eta$, i.e. $D_{KL}(\pi^{k+1}(\cdot|s), \pi^k(\cdot|s)) \leq \mathcal{O}(\eta)$ for all $s \in \sspace$. Thanks to this \emph{slow changing} property, we can prove the following bound on \eqref{eq:main_shift}.
\begin{restatable}{lemma}{shiftlemma}
        \label{lemma:shift}
For the sequence of policies $\bc{\pi^k}^K_{k=1}$ generated by \Cref{alg:theory_version},  for all $\eta > 0$, it holds that $\eqref{eq:main_shift} \leq \frac{\eta K }{(1-\gamma)^3}.$
\end{restatable}
\begin{remark}
The step size choice for $\eta$ in \Cref{thm:policy_regret} is made to trade off optimally the bounds in \Cref{lemma:local_regret,lemma:shift}.
\end{remark}
Finally, the most technical part of the proof aims at bounding the term \eqref{eq:main_optimism}.
\begin{restatable}{lemma}{lemmaoptimism} \label{lemma:main_optimism}
Let us consider an MDP where $\max_{s,a\in\sspace\times\aspace} \mathrm{supp}(P(\cdot|s,a)) = 2$. For each $k \in [K]$, if the $Q^{k+1}$ in \Cref{alg:theory_version}, are updated according to \eqref{eq:update1} or \eqref{eq:update2},  the iterates produced by \Cref{alg:theory_version} satisfy with probability $1-3\delta$ that
\begin{equation*}
    \eqref{eq:main_optimism} \leq \widetilde{\mathcal{O}}\br{\frac{  \sqrt{K \abs{\sspace}^2 \abs{\aspace} \log (1/\delta)}}{1-\gamma}}.
\end{equation*}

\end{restatable}
\textbf{Proof sketch of \Cref{lemma:main_optimism}}
The proof of this Lemma, leverages that the temporal difference errors $\delta^k(s,a)$ produced by \Cref{alg:theory_version} are positive with high probability 
as shown by the next result\footnote{In the main text, we present the proof for the update in \eqref{eq:update1}. The case of update as in \eqref{eq:update2} is deferred to the Appendix.}. 
\begin{restatable}{corollary}{coroptimism}
    \label{cor:optimism}
    Consider an MDP where $\max_{s,a\in\sspace\times\aspace} \mathrm{supp}(P(\cdot|s,a)) = 2$, then for  $L \geq 36\log \br{\frac{ \abs{\sspace}\abs{\aspace} K}{\delta}}$ it holds that with probability at least $1-\delta$
    \begin{equation*}
    \min_{\ell\in[L]} \widehat{P}^k_\ell V^k(s,a) \leq P V^k(s,a)  ~~~ \forall ~~~s,a\in \sspace\times\aspace,~~~\forall ~~k \in [K].
    \end{equation*} 
\end{restatable}

\Cref{cor:optimism} implies that $-\innerprod{d^{\pi^\star}}{\delta^k} \leq 0$ for all $k \in [K]$ and therefore that $\eqref{eq:main_optimism} \leq \sum^K_{k=1}\innerprod{d^{\pi^k}}{\delta^k}$.
\begin{remark}
The above inequality, it is crucial for obtaining the result. Indeed, it upper bounds $\eqref{eq:main_optimism}$ with the \emph{on-policy} temporal difference errors \footnote{That is the temporal difference errors $\delta^k$ averaged by the learner occupancy measures $d^{\pi^k}$} which are small enough to ensure sublinear regret. To see this (informally) consider two cases. First, let us assume that $d^{\pi^k}$ is relatively large for some action pair. Then, that action pair is expected to be visited often in the rollouts and therefore $\delta^k$ is expected to be small. Vice versa, if $\delta^k$ for a certain state-action pair is large, this means that for that state-action pair $d^{\pi^k}$ is relatively small. Overall, we always expect the product $\innerprod{d^{\pi^k}}{\delta^k}$ to be a small quantity.
Notice that the same arguments could not have been carried out replacing $d^{\pi^k}$ with $d^{\pi^\star}$ because the rollouts used in \Cref{alg:theory_version} are not sampled with $\pi^\star$.
\end{remark}
To formalize the above intuition, we upper bound the temporal difference errors with the inverse of the number of times each state-action pair is visited.
\begin{lemma} \label{lemma:Lsmall}(\textbf{Simplified Version of \Cref{lemma:bounded_optimism_ensemble}} )
Let us consider a binarized MDP with $\abs{\sspace}$ states and discount factor $\gamma$. With probability $1-\delta$, it holds that for all $s,a\in \sspace\times \aspace$ and for all $k \in [K]$,
\begin{align*}
\delta^k(s,a) &\leq \widetilde{\mathcal{O}}\br{\sqrt{\frac{L\abs{\sspace} \log (1/\delta)}{(N^k(s,a) + 1) (1-\gamma)^2}}}.
\end{align*}
\end{lemma}
Therefore, by concentration inequalities and noticing that $s^k_{L^k}, a^k_{L^k} \sim d^{\pi^k}$, it holds that with high probability
\begin{align*}
\sum^K_{k=1}&\innerprod{d^{\pi^k}}{\delta^k} = \widetilde{\mathcal{O}}\br{\sum^K_{k=1} \delta^k(s^k_{L^k}, a^k_{L^k})} \\
&\leq \widetilde{\mathcal{O}}\br{\sum^K_{k=1}\sqrt{\frac{L\abs{\sspace} \log (1/\delta)}{(N^k(s^k_{L^k},a^k_{L^k}) + 1) (1-\gamma)^2}}} \\
&\leq \widetilde{\mathcal{O}}\br{\sqrt{K\sum^K_{k=1}\frac{L\abs{\sspace} \log (1/\delta)}{(N^k(s^k_{L^k},a^k_{L^k}) + 1) (1-\gamma)^2}}}.
\end{align*}
At this point, the proof is concluded by bounding the last sum over $K$ with a standard numerical sequences argument (see \Cref{lemma:count_based}).

\paragraph{Optimal choice of the number of critics network $L$} It is important to notice that \Cref{cor:optimism} and \Cref{lemma:Lsmall} creates a tradeoff for what concerns the optimal choice of the number of critics. In particular, from \Cref{cor:optimism}, $L$ should be chosen large enough to ensure that optimism holds with high enough probability. On the other hand, one can notice that \Cref{lemma:Lsmall} upper bounds the expected on policy temporal difference error as $\mathcal{O}(L)$ therefore a smaller number of critics ensures a tighter bound. All in all, the best choice is the smallest $L$ that ensures optimism with probability at least $1-\delta$, that is $L = 36\log \br{\frac{ \abs{\sspace}\abs{\aspace} K}{\delta}}$.
The tradeoff with respect to the number of critics is also observed in a practical ablation study (see Figures~\ref{fig:best_clip_nn} and~\ref{fig:exploration}) .
\subsection{Proof Sketch of \Cref{thm:reward_regret_bound}}
The proof of this term is considerably easier than the bound of the regret for the policy player because we have exact knowledge of the decision variables domain \footnote{$\mathcal{C}$ is taken to be the $\ell_{\infty}$-ball of radius $1$}.
The first step in the proof is to decompose $(1-\gamma)\mathrm{Regret}_c$ as follows
\begin{align*}
&\sum^K_{k=1} \innerprod{c_{\mathrm{true}} - c^k}{\widehat{d^{\pi^k}} - \widehat{d^\expert}}
+\sum^K_{k=1} \innerprod{c_{\mathrm{true}} - c^k}{ d^{\pi^k}- \widehat{d^{\pi^k}}} \\
&\phantom{=}+\sum^K_{k=1} \innerprod{c_{\mathrm{true}} - c^k}{ \widehat{d^\expert} - d^\expert}.
\end{align*}
The first term in the decomposition is upper bounded by $\mathcal{O}(\sqrt{K})$ via a standard online gradient descent analysis \cite{Zinkevich2003}.
Since $\widehat{d^{\pi^k}}$ is an unbiased estimate of the learner occupancy measure, the second term in the decomposition is the sum of a martingale difference sequence. Therefore, an application of the Azuma-Hoeffding inequality ensures that this term grows as $\widetilde{\mathcal{O}}\br{\log(1/\delta)\sqrt{K}}$ with probability at least $1-\delta$.

Finally, the last term is bounded as $\widetilde{\mathcal{O}}\br{K \log(1/\delta)\abs{\mathcal{D}_\expert}^{-1/2}}$ with probability at least $1-\delta$. This is done, proving that for the empirical average estimators for the expert occupancy measure it holds that $\norm{d^\expert-\widehat{d^\expert}}_1 \leq \abs{\mathcal{D}_\expert}^{-1/2} \log(1/\delta)$ with probability at least $1-\delta$. A union bound concludes the proof of \Cref{thm:reward_regret_bound}. The formal proof is deferred to the Appendix.
\section{SOAR for continuous state and actions problems.}
In this section, we explain how \Cref{alg:meta} is instantiated in imitation learning problems with continuous states and action spaces, which therefore requires neural networks to approximate the value function and policy updates.
Since in our analysis for the tabular case, we need to use multiplicative weights/softmax updates, we decided to use SAC, which is an approximation of such updates in the continuous state-action setting.

However, the standard SAC keeps only one network, often called the critic network, to estimate the $Q$ values.
On the other hand, we use a pair of them to avoid the excessive overestimation noticed in Double DQN \cite{vanhasselt2015deepreinforcementlearningdouble}.
Since it uses only one pair of critics, SAC cannot achieve optimism reliably with high probability. 

To fix this issue, we consider multiple critics and we used as an optimistic estimate the mean minus the standard deviation of the ensemble as explained in \Cref{alg:optQnn}. In addition, the standard deviation needs to be truncated at a threshold, as was done in the tabular analysis, to avoid the value function estimators growing out of the attainable range.
For any state $s$, the estimated value functions are truncated in the interval $\bs{0, (1-\gamma)^{-1} }$.
\begin{algorithm}
\caption{\textsc{OptimisticQ-NN} \label{alg:optQnn}}
\begin{algorithmic}[1]
\REQUIRE Replay buffer $\mathcal{D}$, 
Estimators $\bc{Q_{\ell}}^L_{\ell=1}$, 
 maximum standard deviation $\sigma$.
    \State $\bc{s_i}^N_{i=1} \gets$ sample observations from $\mathcal{D}$
    \State $a_i \gets \pi(s_i)$
    \State $\bar{Q}(s_i,a_i) = \frac{1}{L} \sum_{\ell=1}^L Q_\ell (s_i,a_i)  $
    \State $\text{std-Q}(s_\ell,a_\ell) = \sqrt{\frac{1}{L} \sum_{\ell=1}^L \left( Q_\ell (s_i,a_i) - \bar{Q}(s_i,a_i) \right)^2}$
    \State $\overline{\text{std-Q}}(s_i,a_i) \gets \text{Clip}(\text{std-Q}(s_i,a_i), 0, \sigma)$.
    \State $Q(s_i,a_i) = \bar{Q}(s_i,a_i) - \overline{\text{std-Q}}(s_i,a_i)$
    \State \textbf{Return:}  $Q(s_i,a_i)$ for all $i=1, \dots, N$.
\end{algorithmic}
\end{algorithm}
Each of the estimators (critics) $\bc{Q_\ell}^L_{\ell=1}$ is trained in the same way (minimizing the squared Bellman error as in standard SAC ) on a different dataset collected by the same actor. That is, on independent identically distributed datasets. For completeness, the SAC critic training is included in \Cref{alg:updateCritics} in \Cref{app:pseudo}.
In the continuous setting, it is clearly not possible to compute the optimistic state-action value at every state-action pair. Thankfully, it suffices to compute the optimistic state action  value function $Q$, invoking the routine $\textsc{OptimisticQ-NN}$, only for the state-actions in a minibatch $\mathcal{D} = \bc{s_i,a_i}^N_{i=1}$.
Indeed, the policy network weights does not require perfect knowledge of $Q$ over $\sspace\times\aspace$ but only an Adam \cite{Kingma:2015} update step on the loss
$\mathcal{L}_\pi = \frac{1}{N} \sum_{i=1}^N \left( -\eta \log \pi(a_i|s_i) + Q(s_i,a_i)  \right).
$

In the next section, we show that for multiple choices of $\textsc{UpdateCost}$ (ML-IRL, CSIL and RKL) replacing the standard SAC critic update routine with $\textsc{OptimisticQ-NN}$ leads to improved performance.

\begin{figure*}[t]
    \centering
\includegraphics[width=\textwidth]{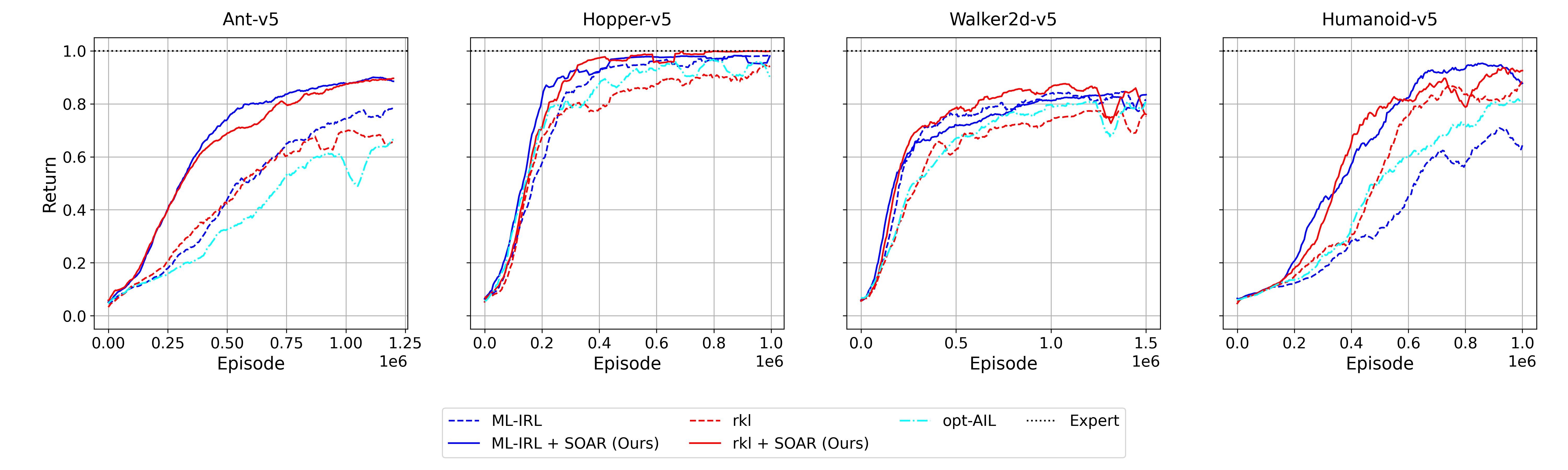}
    \caption{
    \small{\textbf{Experiments from State-Only Expert Trajectories}. 16 expert trajectories, average over 5 seeds, $L=4$ 
    Clipping values $\sigma$ - ML-IRL: [Ant: 10.0, Hopper: 50.0, Walker2d: 0.5, Humanoid: 5.0], rkl: [Ant: 0.8, Hopper: 50.0, Walker2d: 30.0, Humanoid: 100.0]}}
    \label{fig:state_only}
\end{figure*}

\begin{figure*}[t]
    \centering
\includegraphics[width=\textwidth]{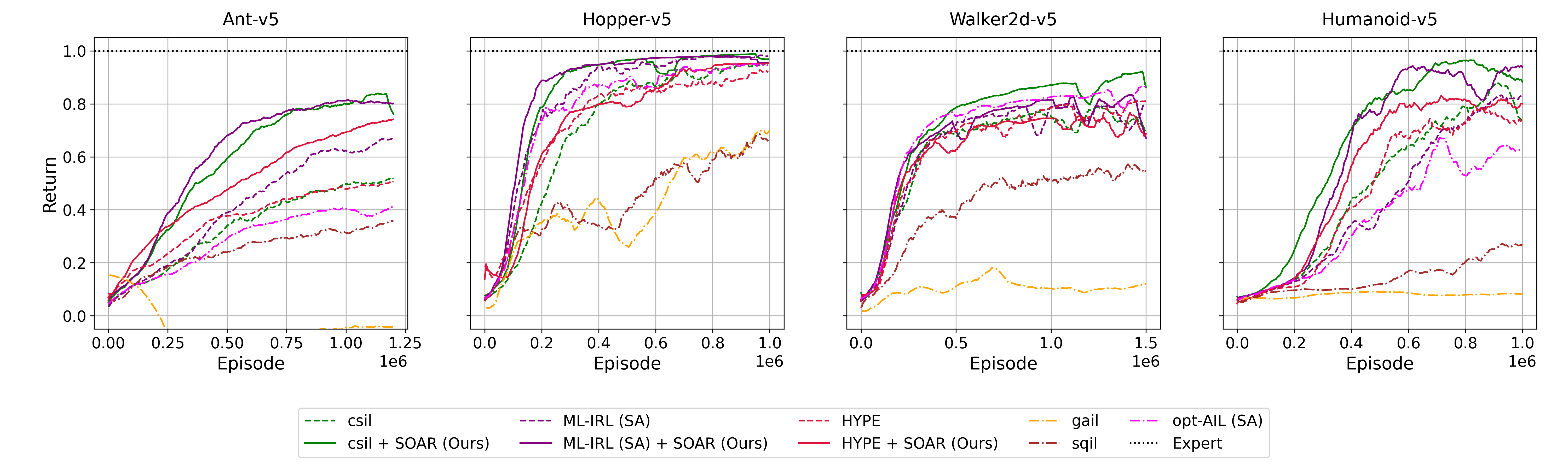}
    \caption{
    \small{\textbf{Experiments from State-Action Expert Trajectories}. 16 expert trajectories, average over 5 seeds, $L=4$.
    Clipping values $\sigma$ - CSIL: [Ant: 10.0, Hopper: 5.0, Walker2d: 0.5, Humanoid: 0.1], ML-IRL(SA): [Ant: 5.0, Hopper: 10.0, Walker2d: 0.5, Humanoid: 50.0]}}
\label{fig:state_actions}
\end{figure*}
\section{Experiments}
We perform experiments for both state only and state action IL on the following MuJoCo \cite{Todorov:2012} environments: Ant, Hopper, Walker2d, and Humanoid.

For the state-only IL setting, we showcase the improvement on RKL \cite{ni2021f} and ML-IRL (State-Only) \cite{zeng2022maximum}. In both cases, we found that using $L=4$ critic networks and an appropriately chosen value for the standard deviation clipping threshold $\sigma$ consistently improves upon the baseline.
In the Appendix \ref{sec:training_proc}, we conduct an ablation study for $L$ and $\sigma$.

We denote our derived algorithms as RKL+SOAR and ML-IRL+SOAR. In addition to observing an improvement over standard RKL and ML-IRL, we outperform the state-only version of the recently introduced OPT-AIL algorithm \cite{xu2024provably} (see Figure~\ref{fig:state_only}) which incorporates an alternative, more complicated, deep exploration technique. 

For the state-action experiments, we plug in the SOAR template on CSIL and the state-action version of ML-IRL. We coined the derived versions CSIL+SOAR and ML-IRL+SOAR (see Appendix~\ref{app:pseudo} for detailed pseudocodes of these algorithms). We also compare with GAIL \cite{Ho:2016b}, SQIL \cite{sqil}, and OPT-AIL. We observe that the exploration mechanism injected by the SOAR principle allows us to achieve reliably superior results (see Figure~\ref{fig:state_actions}).

Further details about the hyperparameters are provided in the Appendix \ref{sec:training_proc}. Moreover, we notice that for all the algorithms in the higher-dimensional and thus more challenging environments (Ant-v5 and Humanoid-v5), the advantage of the SOAR exploration technique becomes more evident. 

The experts trajectory are obtained from policy networks trained via SAC. The expert returns are reported in Table \ref{tab:performance}.

We highlight that the SOAR algorithmic idea can be used also for other imitation learning algorithms based on SAC such as AdRIL \cite{swamy2021moments} and SMILING \cite{wu2024diffusing}.
\begin{table}
    \centering
    \setlength{\tabcolsep}{4pt}
    \caption{\small{Expert returns}}\resizebox{0.45\textwidth}{!}{\begin{tabular}{c|c|c|c|c}
        Method & Ant-v5 & Hopper-v5 & Humanoid-v5 & Walker2d-v5 \\
        \hline
        Expert & 4061.41 & 3500.87 & 5237.48 & 5580.39 \\
        return & $\pm$ 730.58 & $\pm$ 4.33 & $\pm$ 414.69 & $\pm$ 20.30 \\
    \end{tabular}}
    \label{tab:performance}
\end{table}
\subsection{Experiment on a hard exploration task}
An anonymous reviewer pointed out that the MuJoCo benchmark is not the hardest for what concern exploration. This is a very valid suggestion that we address here. This will also allow to understand better the role of the number of critics $L$.
Therefore, to highlight even more the importance of exploration especially in imitation learning from states only we run SOAR-IL in the worst case construction used in the lower bound for the number of environment interaction in \cite[Theorem 19]{moulin2025optimistically}.
This is a simple two states MDP ( a low reward state and a high reward state ) with 20 actions per state. From the high reward state all actions are identical. From the low reward state, all actions are identical but the one chosen by the deterministic expert which has just a slightly higher probability to lead to the high reward state from the low  reward state. 
Even observing the expert state occupancy measure perfectly, it is difficult for the learner to find out which is the action which the expert took. That is because all actions are almost identical but one. 
\begin{figure*}[h]
    \centering
\includegraphics[width=0.5\textwidth]{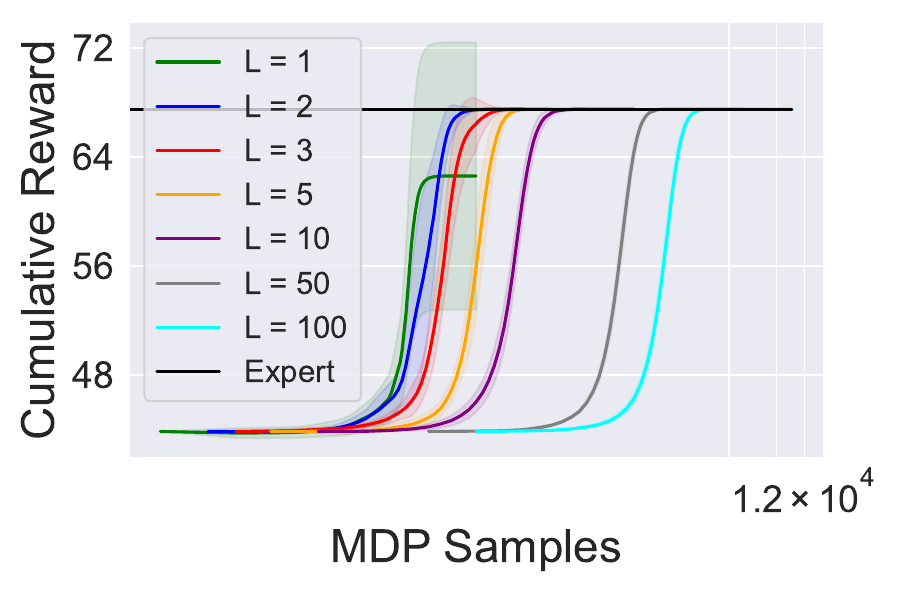}
    \caption{
    \small{\textbf{Ablation for $L$ on hard exploration task}. State only imitation experiment in a hard exploration environment (used in the lower bound from \cite[Theorem 19]{moulin2025optimistically}) }. Results averaged over $5$ seeds, for a dataset of $100$ states sampled from the expert occupancy measure.}
    \label{fig:exploration}
\end{figure*}
We can see that with only 1 network, the mean of the learner does not reach the expert performance and the variance is very high meaning that some seeds are successful and others fail. This is in perfect agreement with \Cref{cor:optimism} which predicts that for low values of the number of critics $L$, the optimistic properties of the critic estimators can not be guaranteed with high probability.
For $L=2$, the environment is solved successfully albeit with a higher variance than the case $L=3$.
Increasing $L$ further leads to worst results in terms of MDP samples needed to solve the task. This is because according to \Cref{lemma:Lsmall} the upper bound on the expected  on policy temporal difference error scales with $L$ so an excessively large $L$ should be avoided. We used $\alpha = 0.5$, $\eta = 4$, and we scaled the standard deviation bonus by $0.001$. 
\section{Conclusions and Open Questions}
While there has been interest in developing heuristically effective exploration techniques in deep RL, the same is not true for deep IL.
For example, even in the detailed study \emph{What matters in Adversarial Imitation Learning ?} \cite{orsini2021matters} the effectiveness of deep exploration techniques is not investigated.
Prior to our work, only few studied the benefits of exploration in imitation learning, mostly in the state-only regime \cite{kidambi2021mobile}. However, their theoretical algorithm uses bonuses that cannot be implemented with neural networks.
Similarly, the recent work \cite{xu2024provably} uses exploration technique in Deep IL but requires solving a complicated non-concave maximization problem.
Our approach is remarkably easier to implement. It achieves convincing empirical results results and enjoys theoretical guarantees.
Moreover, our framework can be expected to be beneficial for any existing or future deep IL algorithm using SAC for policy updates.

\paragraph{Open Questions}
On the theoretical side, we plan to analyze the ensemble exploration technique in the linear MDP case.
From the practical one, we will investigate if the exploration enhanced versions of DQN \cite{osband2016deep,osband2018randomized} can speed up imitation learning from visual input.
Finally, the same idea might find application in the LLM finetuning given the recently highlighted potential of IL for this task \cite{wulfmeier2024imitating,foster2024behavior}.
\newpage
\section*{Acknowledgments}
This work is funded (in part) through a PhD fellowship of the Swiss Data Science Center, a joint venture between EPFL and ETH Zurich.
This work was supported by Hasler Foundation Program: Hasler Responsible AI (project number 21043). Research was sponsored by the Army Research Office and was accomplished under Grant Number W911NF-24-1-0048. This work was supported by the Swiss National Science Foundation (SNSF) under grant number 200021\_205011.
\bibliography{sample, example_paper}
\bibliographystyle{plainnat}

\newpage
\appendix
\onecolumn
\section{Related Works}
\paragraph{IL Theory} 
The first theoretical guarantees obtained for the imitation learning problem dates back to the work of \cite{Abbeel:2004} and \cite{Syed:2007} which notably used the idea of no-regret learning. However, their work requires either knowledge of the environment transitions or they require a suboptimal in the precision parameter $\varepsilon$ amount of expert trajectories to estimate those.
Our theoretical guarantees are in the same setting of previous works like \cite{Shani:2021} and \cite{xu2023provably} which do not require knowledge of the environment transitions a priori but assumes online trajectory access to the environment.
 The main difference that their work focuses on the easier finite horizon setting. Additionally, their exploration techniques only applies to the tabular setting. Indeed, in the MuJoCo experiments in  \cite{Shani:2021}, the authors do not attempt to implement the exploration mechanism required for their theoretical guarantees. 
In a similar way, also \cite{xu2023provably} can not be implemented beyond the tabular setting because it relies on a reward free procedure requiring bonuses proportional to the number of visits to each state action pair.
\cite{ren2024hybrid} suggest an algorithm which does not require exploration but it can not improve upon behavioural cloning in terms of expert trajectories. 
\cite{rajaraman2020toward, rajaraman2021provablybreaking,foster2024behavior} analyze instead offline imitation learning (behavioural cloning) where no additional interaction with the environment is allowed. This setting is more general but it comes at the cost of additional assumptions such as policy realizability or worst depedence on the horizon on the required number of trajectories.
\cite{foster2024behavior} presents an analysis for general policy classes but they require a maximum likelihood oracle which can not be implemented exactly when using neural network function approximation.

There has been also a variety of studies tackling the problem of computationally efficient algorithm with linear function approximation such as \cite{Kamoutsi:2021, viano2022proximal,viano2024imitation,rajaraman2021value,swamy2022minimax}.
However, their proof techniques are strictly depending on the linearity of the dynamics therefore the experiments in continuous control tasks require changes in the algorithmic design. Albeit our guarantees are restricted to the tabular setting, the algorithm can be implemented with no modifications with neural networks.

Several works focus on the setting where expert queries are allowed at any state visited during the MDP interaction \cite{Ross:2010, Ross:2011,swamy2021moments} or that require a generative model for the algorithm updates  \cite{swamy2022minimax}.
Another recent work requires a generative model to sample the initial state of the trajectory from the expert occupancy measure \cite{swamy2023inverse}.
Our algorithm requires sampling only trajectories in the MDP therefore it does not leverage the aforementioned generative model
assumption.
In contrast, the setting of this work matches the most practical one adopted for example in \cite{Ho:2016, Ho:2016b, Fu:2018, Reddy:2020, Dadashi:2021, watson2023coherent, Garg:2021, ni2021f}. In this case, the expert policy can  not be queried and the learner access only a precollected dataset of expert demonstrations.
\paragraph{Theory for IL from States Only}This setting has been firstly studied in \cite{sun2019provably} in the finite horizon setting and with general function approximation their work does not use exploration mechanism. However their work requires an additional realizability assumption of the expert value function, it can only learn a difficult to store and deploy non stationary policy and provides suboptimal guarantees on $\tau_E$ in terms of the horizon dependence.

The follow up from \cite{arora2020provable}, still requires the realizability of the state value function which is not needed in our work.
The work of \citet{kidambi2021mobile}  uses the idea of exploration in state only finite horizon imitation learning. Their analysis for tabular MDP gives a bound on $K$ which has a worst horizon dependence and it requires the design of exploration bonuses tight to the structural properties of the MDP. Therefore, their NN experiments requires an empirical approximation of such bonuses while the SOAR framework applies naturally.

\citet{wu2024diffusing} imposes expert score function realizability and that the expected state norm remains bounded during learning. The algorithm has provable guarantees but it requires an expensive \emph{RL in the loop} routine that we avoid in our work.
\paragraph{Exploration Techniques in Deep RL}
Ensemble of $Q$ networks has also been used for training stabilization \cite{anschel2017averaged}. \cite{zhang2025beta} introduces exploration technique based on multiple actors.
\citet{ciosek2019better} does not have theoretical guarantees but it uses the idea of constructing an optimistic critic using mean plus standard deviation but only to define an exploratory policy with which collecting data. Our approach instead maintains only one actor policy which is updated with the optimistic $Q$ estimate. \cite{parker2020effective,lyu2022efficient} exploration with ensemble of actors rather than critics. \cite{kurutach2018model, chua2018deep} uses an ensemble of networks trained to learn the transition model to improve the sample complexity in model based RL. 
\cite{henaff2022exploration} learns instead an inverse dynamics model and via an encoder and decoder model and uses the features output by the encoder to compute elliptical potential bonuses which are standard in linear bandits \cite{Abbasi-Yadkori:2011}.
\citet{moskovitz2021tactical} improved TD-3 \cite{fujimoto2018addressing} using an ensemble of critics and a bandit algorithm to find an aggregation rule balancing well the amount of optimism required by online exploration and pessimism required by off policy algorithms such as TD-3.  

In addition, there are several deep RL work that takes a bayesian point of view to the problem, these algorithms often achieve remarkable performance but the algorithm implemented with deep networks requires usually adjustments creating a mismatch compared to the provable algorithms in the tabular case. 
Among those \cite{luis2023model,zhou2020deep,o2018uncertainty} use the Bellman equation for the state value function variance to train a network (dubbed $U$ network) that models the uncertainty of the network predicting the $Q$ values. They respectively prove that this trick improves the performances of SAC, PPO \cite{Schulman:2017} and DQN\cite{Mnih:2015}. \cite{curi2020efficient} uses the model uncertainty estimate in the update of the actor.  

Moreover, building on the theoretical analysis of PSRL \cite{osband2014near} and RLSVI \cite{osband2016generalization} that show sublinear bayesian regret bound. At any step, these algorithms sample from a posterior distribution either an MDP where to plan or a value function to follow greedly at each step. Between one step and the other the posterior is updated given the new data.
While the theorical analysis in the above works prescribe a randomization at the value function parameters level, in the deep RL version, dubbed Boostrapped DQN \cite{osband2016deep},  the perturbation is performed implicitly maintaining a set of $Q$ networks and sampling uniformly at each round according to which network the agent chooses the greedy action. 
\cite{chen2017ucb}
improved upon Bootstrapped DQN using an aggregation rule. That is acting greedy with respect to the mean plus standard deviation of the $q$ ensemble.
\citet{osband2018randomized} further builds on this idea adding a differ prior to each network in the ensemble to increase diversity. Finally, \citet{osband2023approximate} replaces the uniform sampling in \cite{osband2016deep} with a learned distribution with an epistemic network \cite{osband2023epistemic}.

Furthermore, motivated by the bayesian regret bound proven in \cite{o2021variational} in the tabular case and the one in  \cite{o2023efficient}, \cite{tarbouriech2024probabilistic} proves a regret bound in the function approximation setting and showcased convincing performance in the Atari benchmark.
Their algorithm requires to know the variance of the cost posterior distribution which is not available in the neural network experiments. Therefore, it is estimated using the standard deviation of an ensemble of cost network. In our work, we use an ensemble of $Q$ networks and not cost networks.

Additionally,
\cite{ishfaq2021randomized} analyzed ensemble exploration techniques in the general function approximation setting. Their ensemble consists of different critics trained on the same state actions dataset but with rewards perturbed with a gaussian random vector. \cite{ishfaq2023provable,ishfaq2024more} looked at efficient implementation of Thompson sampling in Deep RL and obtained convincing results in Atari and providing guarantees for linear MDPs and general function approximation respectively.
Moreover, \cite{ishfaq2025langevin} extended the above results for continuous action spaces.
Unfortunately, these methods do not apply directly to imitation learning because they require a fixed reward function.

\paragraph{Exploration techniques in Deep IL}
As mentioned only few works investigated exploration techniques in Deep IL. Apart from the 
previously mentioned works, \cite{yu2020intrinsic} adopts a model based approach and used exploration bonuses based on prediction error of the next observed state (a.k.a. curiosity driven exploration \cite{pathak2017curiosity,burda2018large}). Finally we notice that ensembles have been used in IL theory IL also for goals different to exploration. In particular, \cite{swamy2022minimax} partitioned the expert dataset in two subdataset and show that these technique allows for improved expert sample complexity bounds when the expert is deterministic.

\paragraph{State-only imitation learning}
\citet{torabi2018generative} tackled the problem of imitation learning from states only modifying the discriminator of GAIL \cite{Ho:2016b} to take as input state next state pairs instead of state action pairs. Further practical improvements have been proposed in \cite{zhu2020off} that allows for the use of off-policy data. The works \cite{yang2019imitation,nair2017combining,pathak2018zero,radosavovic2021state} use the idea of an inverse dynamic model while \cite{edwards2019imitating,ganai2023learning} develops a practical algorithm aiming at estimating the forward dynamic model. Furthermore, \cite{torabi2018behavioral} introduces a twist in behavioral cloning using inverse dynamic modelling to make it applicable to state only expert datasets.  A comprehensive literature review can be found in \cite{torabi2019recent}. 
More recently, features/state only imitation learning has found application in non markovian decision making problems \cite{qin2024learning}. \citet{sikchi2022ranking} introduce an algorithm that takes advantage of an offline ranker between trajectories  to get strong empirical results in LfO setting.
 Another line of works \cite{gupta2017learning,sermanet2018time,liu2019state,viano2021robust,viano2022robust,gangwani2020state,cao2021learning,gangwani2022imitation} motivate imitation learning from observation alone arguing that the expert providing the demonstrations and the learner acts in slightly different environments. In \cite{kim2022lobsdice,sikchi2024dual}, the authors proposed convex programming based methods to imitate an expert policy from  expert state only demonstration and auxiliary arbitrary state action pairs. Several works \cite{ni2021f,kim2022demodice,ma2022versatile,yu2023offline} introduce empirical methods to minimize an $f$-divergence between expert and learner state occupancy measure. Complementary, \cite{chang2023imitation} minimizes the Wasserstein distance between expert and learner state occupancy measure. Their numerical results are convincing but no sample complexity bounds are provided. Convincing results have been obtained also in \cite{chang2024adversarial} that uses the idea of boosting and in \cite{wu2024diffusing} which uses a diffusion models inspired loss to update the cost.

\section{Theoretical Analysis}

\subsection{Upper bounding the policy regret}
\coroptimism*
\begin{proof}
Let us fix a state-action-next state triplet $s,a,s'$, a batch index $\ell \in [L]$ and an iteration index $k \in [K]$.
Then, we consider the following stochastic estimator for the probability of transitioning to $s'$ from state $s$ taking action $a$.
\begin{align*}
\widehat{P}^k_\ell(s'|s,a) &=\frac{N^k_\ell(s,a,s')}{N^k_\ell(s,a) + 2} = \frac{\sum_{\bar{s},\bar{a},\bar{s}'\in \mathcal{R}^k_\ell} \mathds{1}_{\bc{\bar{s},\bar{a},\bar{s}' = s,a,s'}}}{N^k_\ell(s,a) + 2} \\
&= \frac{\sum_{\bar{s},\bar{a},\bar{s}'\in \mathcal{R}^k_\ell : \bar{s},\bar{a} = s,a } \mathds{1}_{\bc{\bar{s}' = s'}}}{N^k_\ell(s,a) + 2}.
\end{align*}
Notice that in above estimators the denominator corresponds to the number of visits of the pair $s,a$ up to the time $k \in [K]$ within the batch $\ell\in[L]$, i.e. $N^k_\ell(s,a)$,  increased by $2$ for technical reasons. At the numerator instead we have the sum of $N^k_\ell(s,a)$ indicators functions which equals one when the state following the state action pair $s,a$ is equal to $s'$. Each of this indicator is a random variable distributed 
according to a Bernoulli random variable with mean $P(s'|s,a)$.
At this point, we can use the technique introduced in \cite{cassel2024batch}. In particular, we will show that
\begin{equation*}
\mathbb{P}\bs{ \widehat{P}^k_\ell(s'|s,a) \leq P(s'|s,a)} \geq \frac{1}{4} ~~~ \forall s,a,s' \in \sspace\times\aspace\times\mathrm{children}(s), \forall k \in [K], \forall \ell\in[L].
\end{equation*}
We distinguish $3$ cases: $N^k_\ell(s,a) = 0$, $N^k_\ell(s,a) = 1$, $N^k_\ell(s,a) \geq 2$.
If $N^k_\ell(s,a) = 0$, then, we have that
\begin{align*}
    \widehat{P}^k_\ell(s'|s,a) = 0 \leq P(s'|s,a).
\end{align*}
If $N^k_\ell(s,a) = 1$, we distinguish two cases. If $P(s'|s,a) \geq \frac{1}{3}$, then
\begin{align*}
\widehat{P}^k_\ell(s'|s,a) \leq \frac{1}{3} \leq P(s'|s,a).
\end{align*}
Otherwise, if $N^k_\ell(s,a) = 1$, and $P(s'|s,a) \leq \frac{1}{3}$ then
\begin{align*}
\mathbb{P}\bs{
\frac{\mathds{1}_{\bc{\bar{s}' = s'}}}{3} \leq P(s'|s,a) 
} = \mathbb{P}\bs{
\mathds{1}_{\bc{\bar{s}' = s'}} = 0} = 1 - P(s'|s,a) \geq \frac{2}{3} \geq \frac{1}{4}.
\end{align*}
Finally, for $N^k_\ell(s,a) \geq 2$,we have that for $P(s'|s,a) \geq 1 - \frac{1}{N^k_\ell(s,a)}$ it holds that
\[
\widehat{P}^k_\ell(s'|s,a) \leq \frac{N^k_\ell(s,a)}{N^k_\ell(s,a)+2} = 1 - \frac{2}{N^k_\ell(s,a) + 2} \leq 1 - \frac{1}{N^k_\ell(s,a)} \leq P(s'|s,a),
\]
where the last inequality holds for $N^k_\ell \geq 2$.
Otherwise, for $P(s'|s,a) \leq 1 - \frac{1}{N^k_\ell(s,a)}$ we can apply \citep[Lemma 2]{cassel2024batch} (adapted from \citep[Corollary 1]{wiklund2023another} ) to obtain 
\begin{equation*}
\mathbb{P}\bs{\widehat{P}^k_\ell(s'|s,a) \leq P(s'|s,a)} \geq 
\mathbb{P}\bs{ \sum_{\bar{s},\bar{a},\bar{s}'\in \mathcal{R}^k_\ell : \bar{s},\bar{a} = s,a } \mathds{1}_{\bc{\bar{s}' = s'}} \leq N^k_\ell(s,a) P(s'|s,a)} \geq \frac{1}{4}.
\end{equation*}
At this point notice that for any positive vector $V \in [0, (1-\gamma)^{-1}]^{\abs{\sspace}}$, it holds that
\begin{align*}
\mathbb{P}\bs{\widehat{P}^k_\ell(s'|s,a) \leq P(s'|s,a)~~\forall s'\in \mathrm{children}(s)} &= \mathbb{P}\bs{\widehat{P}^k_\ell(s'|s,a)  V(s') \leq P(s'|s,a) V(s')~~\forall s'\in \mathrm{children}(s)} \\ & \leq \mathbb{P}\bs{\sum_{s'\in\sspace}\widehat{P}^k_\ell(s'|s,a)  V(s') \leq \sum_{s'\in\sspace} P(s'|s,a) V(s')}
\end{align*}
where the inequality holds because of the following implication between events
\begin{equation*}
\widehat{P}^k_\ell(s'|s,a)  V(s') \leq  P(s'|s,a) V(s') ~~\forall s'\in \mathrm{children}(s)
 \implies \sum_{s'\in \mathrm{children}(s)}\widehat{P}^k_\ell(s'|s,a)  V(s') \leq \sum_{s'\in \mathrm{children}(s)} P(s'|s,a) V(s').
\end{equation*}
Moreover, since the estimation at each state $s'$ is independent we have that
\begin{align*}
\mathbb{P}\bs{\widehat{P}^k_\ell(s'|s,a) \leq P(s'|s,a)~~\forall s'\in \mathrm{children}(s)} = \prod_{s'\in \mathrm{children}(s)} \mathbb{P}\bs{\widehat{P}^k_\ell(s'|s,a) \leq P(s'|s,a)} \geq \br{\frac{1}{4}}^{\abs{\mathrm{children}(s)}}
\end{align*}
Then, for any $s,a \in \sspace \times \aspace$ is concluded by the following chain of inequalities
\begin{align*}
\mathbb{P}&\bs{\min_{\ell\in[L]}\sum_{s'\in \mathrm{children}(s)}\widehat{P}^k_\ell(s'|s,a)  V(s') \geq \sum_{s'\in \mathrm{children}(s)} P(s'|s,a) V(s')} \\&= \mathbb{P}\bs{\sum_{s'\in\mathrm{children}(s)}\widehat{P}^k_\ell(s'|s,a)  V(s') \geq \sum_{s'\in \mathrm{children}(s)} P(s'|s,a) V(s')~~ \forall \ell \in [L]} \\
&= \prod_{\ell \in [L]}\br{1 - \mathbb{P}\bs{\sum_{s'\in \mathrm{children}(s)}\widehat{P}^k_\ell(s'|s,a)  V(s') \leq \sum_{s'\in \mathrm{children}(s)} P(s'|s,a) V(s')}} 
\\
&\leq \prod_{\ell\in[L]}\br{ 1-\mathbb{P}\bs{\widehat{P}^k_\ell(s'|s,a)  \leq P(s'|s,a)~~\forall s'\in \mathrm{children}(s)} } \\
&\leq \prod_{\ell\in[L]}\br{ 1- \frac{1}{4}^{\abs{ \mathrm{children}(s)}} } \leq e^{-L \log \br{\frac{4^{\abs{ \mathrm{children}(s)}}}{4^{\abs{\mathrm{children}(s)}} - 1}}}.
\end{align*}
Therefore,
choosing $L \geq \frac{\log \br{1/\delta}}{\log \br{\frac{4^{\abs{ \mathrm{children}(s)}}}{4^{\abs{\mathrm{children}(s)}} - 1}}}$, we ensure that
\begin{align*}
\mathbb{P}\bs{\min_{\ell\in[L]} \widehat{P}^k_\ell V^k(s,a) \geq P V^k(s,a)} \leq \delta.
\end{align*}
For $\abs{\mathrm{children(s)}} = 2$, we have that $\br{\log \br{\frac{4^{\abs{ \mathrm{children}(s)}}}{4^{\abs{\mathrm{children}(s)}} - 1}}}^{-1}\leq 36$, therefore with a union bound over the sets $\sspace,\aspace$ and $[K]$ we conclude that
for $L \geq 36 \log \br{\frac{\abs{\sspace}\abs{\aspace} K}{\delta}}$, it holds that
\begin{align*}
\mathbb{P}\bs{\min_{\ell\in[L]} \widehat{P}^k_\ell V^k(s,a) \leq P V^k(s,a) ~~\forall ~~s,a,k \in \sspace\times\aspace\times [K]} \geq 1-\delta.
\end{align*}
\end{proof}

\subsection{Policy Regret Decomposition}

\thmpolicyregret*
\begin{proof}
The theorem is proven  with the following regret decomposition
in virtue of \Cref{lemma:infinite_extendend_pdl} already presented in the main text. Denoting $\delta^{k}(s,a) \triangleq \cost^k(s,a) + \gamma P V^k(s,a) - Q^{k+1}(s,a)$ and $g^k(s,a) \triangleq  Q^{k+1}(s,a) - Q^{k}(s,a)$
\begin{align}
(1 - \gamma) & \mathrm{Regret}_\pi(K;\pi^\star) = \sum^K_{k=1} \mathbb{E}_{s\sim d^{\pi^\star}}\bs{
    \innerprod{Q^k(s,\cdot)}{\pi^k(s) - \pi^\star(s)}} \quad \tag{BTRL} \label{eq:OMD}\\&\phantom{\leq}+ \sum^K_{k=1}  \sum_{s,a} \bs{d^{\pi^k}(s,a) - d^{\pi^\star}(s,a)}\cdot\bs{\delta^k(s,a)} \tag{Optimism} \label{eq:optimism}
    \\&\phantom{\leq}+ \sum^K_{k=1} \mathbb{E}_{s,a\sim d^{\pi^k}}\bs{g^k(s,a)} - \sum^K_{k=1} \mathbb{E}_{s,a\sim d^{\pi^\star}}\bs{g^k(s,a)} \quad \label{eq:shift} \tag{Shift}
\end{align}
At this point, we bound each term with the following Lemmas, we obtain
\begin{align*}
&\sqrt{\frac{\log \abs{\aspace} K }{ (1-\gamma)^5}} + \widetilde{\mathcal{O}}\br{\frac{ \sqrt{ K \abs{\sspace}^2 \abs{\aspace} \log (1/\delta)}}{(1-\gamma)^2}} \\&\leq \widetilde{\mathcal{O}}\br{ \sqrt{ \frac{K \abs{\sspace}^2 \abs{\aspace} \log (1/\delta)}{(1-\gamma)^5}}}.
\end{align*}
\end{proof}
\textbf{Bound on \eqref{eq:main_OMD}}
Then, we continue bounding the first term invoking the following lemma.
\localregret*
\begin{proof}

    \begin{align}
\sum^K_{k=1} \mathbb{E}_{s\sim d^{\pi^\star}_{\rho_{e_k}}}\bs{
    \innerprod{Q^k(s,\cdot)}{\pi^k(s) - \pi^\star(s)}} &= \sum_{k \in K} \mathbb{E}_{s\sim d^{\pi^\star}}\bs{
    \innerprod{Q^k(s,\cdot)}{\pi^k(s) - \pi^\star(s)}} \nonumber \\
    &=\mathbb{E}_{s\sim d^{\pi^\star}}\bs{ \sum_{k \in K} 
    \innerprod{Q^k(s,\cdot)}{\pi^k(s) - \pi^\star(s)}} \label{eq:first}
    \end{align}
Then,  applying a standard regret bound for Be the Regularized Leader (BTRL) it holds that for all $s \in \sspace$
\begin{align*}
    \sum_{k \in K} 
    \innerprod{Q^k(s,\cdot)}{\pi^k(s) - \pi^\star(s)} &= \frac{\log \abs{\aspace}}{\eta}
\end{align*}
Then, plugging into \eqref{eq:first} we conclude that
\begin{equation*}
\eqref{eq:main_OMD} \leq \frac{\log \abs{\aspace}}{\eta}.
\end{equation*}
\end{proof}
\noindent\textbf{Bound on \eqref{eq:main_shift}}
We follow the idea from \cite{moulin2023optimistic} of controlling this term proving that two consecutive policies will have occupancy measures within a $\mathcal{O}(\eta)$ total variation distance. 
\shiftlemma*
\begin{proof}
Let us denote $Q_{\max} = (1-\gamma)^{-1}$
\begin{align*}
    \sum^K_{k=1} \innerprod{d^{\pi^k}}{Q^{k} - Q^{k+1}} &=  \br{ \innerprod{d^{\pi^1}}{Q^1} + \sum^{K-1}_2 \innerprod{Q^k}{d^{\pi^k} - d^{\pi^{k-1}} } - \innerprod{d^{\pi^K}}{Q^K}} \\
&\leq  Q_{\max} + \frac{\eta Q^2_{\max} (K-2) }{(1-\gamma)} \leq \frac{\eta Q^2_{\max} (K-1) }{(1-\gamma)} .
\end{align*}
In the first inequality, we used Lemma~\ref{lemma:slow}.
Finally, noticing that
\begin{equation*}
- \sum^K_{k=1} \innerprod{d^{\pi^\star}}{Q^{k} - Q^{k+1}} = \innerprod{d^{\pi^\star}}{Q^K - Q^1} \leq Q_{\max}
\end{equation*}
and that
\begin{equation*}
    \eqref{eq:main_shift} = \sum^K_{k=1} \innerprod{d^{\pi^k}}{Q^{k} - Q^{k+1}} - \sum^K_{k=1} \innerprod{d^{\pi^\star}}{Q^{k} - Q^{k+1}}
\end{equation*}
allows us to conclude the proof summing the two bounds.
    \end{proof}
\textbf{Bound on \eqref{eq:main_optimism}}
\lemmaoptimism*
\begin{proof}
For the optimism term we can observe that using the update of the $Q$ values, we have that
$$
\delta^k(s,a) = \gamma (P V^k(s,a) - \min_{\ell \in [L]} \widehat{P}^k_{\ell} V^k(s,a))
$$
Therefore, in virtue of \Cref{cor:optimism} with probability $1-\delta$ it holds that
\begin{equation*}
\delta^k(s,a) \geq 0 ~~~\forall ~~~s,a \in \sspace\times\aspace.
\end{equation*}
Therefore, with probability $1-\delta$, we have that
\begin{align*}
\sum^K_{k=1}  \sum_{s,a} \bs{d^{\pi^k}(s,a) - d^{\pi^\star}(s,a)}\cdot\bs{\delta^k(s,a)} &\leq \sum^K_{k=1}  \sum_{s,a} d^{\pi^k}(s,a) \delta_k(s,a) \\
&=\gamma\sum^K_{k=1} \mathbb{E}_{s,a\sim d^{\pi^k}} \bs{P V^k(s,a) - \min_{\ell \in [L]} \widehat{P}^k_{\ell} V^k(s,a)}
\end{align*}
At this point, using \Cref{lemma:bounded_optimism_ensemble} and  a union bound we have that with probability $1-2\delta$, it holds that
\begin{align*}
\sum^K_{k=1}  \sum_{s,a} \bs{d^{\pi^k}(s,a) - d^{\pi^\star}(s,a)}\cdot\bs{\delta^k(s,a)} &\leq 
\gamma\sum^K_{k=1} \mathbb{E}_{s,a\sim d^{\pi^k}} \bs{\sqrt{\frac{\abs{\sspace}}{(N^k(s,a)/L + 1) (1-\gamma)^2}\log \br{\frac{\abs{\sspace}\abs{\aspace}  L K^2(K+1) }{(1-\gamma)\delta}}}} \\
&\phantom{=}+  \gamma\sum^K_{k=1} \mathbb{E}_{s,a\sim d^{\pi^k}} \bs{\frac{2}{(N^k(s,a)/L + 1)(1-\gamma)}} + 4 \\
&\leq \sqrt{K \sum^K_{k=1} \mathbb{E}_{s,a\sim d^{\pi^k}} \bs{\frac{\abs{\sspace}}{(N^k(s,a)/L + 1) (1-\gamma)^2}\log \br{\frac{\abs{\sspace}\abs{\aspace}  L K^2(K+1) }{(1-\gamma)\delta}}}} \\
&\phantom{=}+  \gamma\sum^K_{k=1} \mathbb{E}_{s,a\sim d^{\pi^k}} \bs{\frac{2}{(N^k(s,a)/L + 1)(1-\gamma)}} + 4 \\
\end{align*}
At this point, since $L^k$ is geometrically distributed for every $k \in [K]$ it holds that $L^k \leq L_{\max} := \frac{\log (K/\delta)}{(1-\gamma)}$ for all $k \in [K]$ with probability $1-\delta$.
Therefore, we can invoke \Cref{lemma:count_based}  to bound $\sum^K_{k=1} \mathbb{E}_{s,a\sim d^{\pi^k}} \bs{\frac{2}{(N^k(s,a)/L + 1)}}$. Another union bound ensures that with probability $1 - 3\delta$, we have that
\begin{align*}
    \sum^K_{k=1}  \sum_{s,a} & \bs{d^{\pi^k}(s,a) - d^{\pi^\star}(s,a)}\cdot  \bs{\delta^k(s,a)} \\&\leq \frac{1}{1-\gamma}\sqrt{K \abs{\sspace} \log \br{\frac{\abs{\sspace} \abs{\aspace} L K^2(K+1)}{(1-\gamma)\delta}} (2 L \abs{\sspace} \abs{\aspace} \log \br{K L_{\max}} + 4 \log \br{2K/\delta} ) } \\
    &\phantom{=}+\frac{2}{1-\gamma} (2 L \abs{\sspace} \abs{\aspace} \log \br{K L_{\max}} + 4 \log \br{2K/\delta}) \\
&=\widetilde{\mathcal{O}}\br{\frac{  \sqrt{K \abs{\sspace}^2 \abs{\aspace} \log (1/\delta)}}{1-\gamma}}.
\end{align*}
where the $\widetilde{\mathcal{O}}$ notation hides logarithmic factors in $K, 1-\gamma, \abs{\sspace}$ and $\abs{\aspace}$.

Now, we prove the part of the Theorem that considers the update for $Q^{k+1}$ given in \eqref{eq:update2}. Under this update, we have that
$$ \delta^k(s,a) = \gamma \br{PV^k(s,a)  -  \max\bs{ \frac{1}{L} \sum^L_{\ell=1} \widehat{P}^k_\ell V^k(s,a) - \sqrt{\sum^L_{\ell=1} \br{
    \widehat{P}^k_\ell V^k(s,a)
    - \frac{1}{L} \sum^L_{\ell=1} \widehat{P}^k_\ell V^k(s,a)}^2},0}}$$
Then, applying Samuelson's inequality \Cref{lemma:samuelson}, we have that
$$
\min_{\ell\in[L]} \widehat{P}^k_\ell V^k(s,a) \geq \frac{1}{L} \sum^L_{\ell=1} \widehat{P}^k_\ell V^k(s,a) - \sqrt{\sum^L_{\ell=1} \br{
    \widehat{P}^k_\ell V^k(s,a)
    - \frac{1}{L} \sum^L_{\ell=1} \widehat{P}^k_\ell V^k(s,a)}^2}
$$
moreover, it holds that
$$
\min_{\ell\in[L]} \widehat{P}^k_\ell V^k(s,a) \geq 0
$$
Therefore,
\[
\min_{\ell\in[L]} \widehat{P}^k_\ell V^k(s,a) \geq \max \bs{\frac{1}{L} \sum^L_{\ell=1} \widehat{P}^k_\ell V^k(s,a) - \sqrt{\sum^L_{\ell=1} \br{
    \widehat{P}^k_\ell V^k(s,a)
    - \frac{1}{L} \sum^L_{\ell=1} \widehat{P}^k_\ell V^k(s,a)}^2}, 0}
\]
which implies 
$$ \delta^k(s,a) \geq \gamma \br{PV^k(s,a)  -  \min_{\ell\in[L]} \widehat{P}^k_\ell V^k(s,a)} \geq 0
$$
where the last inequality holds with probability $1-\delta$ thanks to \Cref{cor:optimism}.
Therefore, with probability $1-\delta$, we have that
\begin{align*}
\sum^K_{k=1}  &\sum_{s,a} \bs{d^{\pi^k}(s,a) - d^{\pi^\star}(s,a)}\cdot\bs{\delta^k(s,a)} \leq \sum^K_{k=1}  \sum_{s,a} d^{\pi^k}(s,a) \delta_k(s,a) \\
&\leq\gamma\sum^K_{k=1} \mathbb{E}_{s,a\sim d^{\pi^k}} \bs{\frac{1}{L}\sum^L_{\ell=1} \br{P V^k(s,a) -  \widehat{P}^k_{\ell} V^k(s,a)} + \sqrt{\sum^L_{\ell=1} \br{\widehat{P}^k_\ell V^k (s,a) - \frac{1}{L}\sum^L_{\ell=1} \widehat{P}^k_{\ell} V^k(s,a)}^2}} ,
\end{align*}
where the last inequality holds removing the maximum.
For the first term, we can use \Cref{lemma:bounded_optimism_ensemble} and continue as in the previous proof to show that with probability $1-\delta$
\begin{equation*}
    \gamma\sum^K_{k=1} \mathbb{E}_{s,a\sim d^{\pi^k}} \bs{\frac{1}{L}\sum^L_{\ell=1} \br{P V^k(s,a) -  \widehat{P}^k_{\ell} V^k(s,a)}} \leq \widetilde{\mathcal{O}}\br{\frac{  \sqrt{K \abs{\sspace}^2 \abs{\aspace} \log (1/\delta)}}{1-\gamma}}.
\end{equation*}
So, we are left with bounding the term 
\begin{equation*}
\gamma\sum^K_{k=1} \mathbb{E}_{s,a\sim d^{\pi^k}} \bs{\sqrt{\sum^L_{\ell=1} \br{\widehat{P}^k_\ell V^k (s,a) - \frac{1}{L}\sum^L_{\ell'=1} \widehat{P}^k_{\ell'} V^k(s,a)}^2}}
\end{equation*}
To this end, by Jensen's inequality we have that
\begin{align*}
\gamma\sum^K_{k=1} & \mathbb{E}_{s,a\sim d^{\pi^k}} \bs{\sqrt{\sum^L_{\ell=1} \br{\widehat{P}^k_\ell V^k (s,a) - \frac{1}{L}\sum^L_{\ell'=1} \widehat{P}^k_{\ell'} V^k(s,a)}^2}} \\&\leq \gamma\sum^K_{k=1} \mathbb{E}_{s,a\sim d^{\pi^k}} \bs{\sqrt{\frac{1}{L}\sum^L_{\ell'=1}\sum^L_{\ell=1} \br{\widehat{P}^k_\ell V^k (s,a) -  \widehat{P}^k_{\ell'} V^k(s,a)}^2}} \\
\\&\leq \gamma\sum^K_{k=1} \mathbb{E}_{s,a\sim d^{\pi^k}} \bs{\sqrt{\frac{2}{L}\sum^L_{\ell'=1}\sum^L_{\ell=1} \br{\widehat{P}^k_\ell V^k (s,a) -  P V^k(s,a)}^2}} \\
&\leq \gamma\sum^K_{k=1} \mathbb{E}_{s,a\sim d^{\pi^k}} \bs{\sqrt{2\sum^L_{\ell=1} \br{\widehat{P}^k_\ell V^k (s,a) -  P V^k(s,a)}^2}} \\
\end{align*}
At this point, notice that invoking \Cref{lemma:bounded_optimism_ensemble}, we have that for all $\ell \in [L]$ with probability $1-\delta$
\begin{align*}
&\br{\widehat{P}^k_\ell V^k (s,a)  -  P V^k(s,a)}^2 \\ &\leq \br{\sqrt{\frac{\abs{\sspace}}{(N^k(s,a)/L + 1) (1-\gamma)^2}\log \br{\frac{\abs{\sspace}\abs{\aspace}  L K^2(K+1) }{(1-\gamma)\delta}}} + \frac{2}{K} +  \frac{2}{(N^k(s,a)/L + 1)(1-\gamma)}}^2 \\
&\leq \frac{3\abs{\sspace}}{(N^k(s,a)/L + 1) (1-\gamma)^2}\log \br{\frac{\abs{\sspace}\abs{\aspace}  L K^2(K+1) }{(1-\gamma)\delta}} + \frac{12}{K^2} + \frac{12}{(N^k(s,a)/L + 1)^2(1-\gamma)^2} \\
&= \widetilde{\mathcal{O}}\br{\frac{3\abs{\sspace}\log \br{\frac{1}{\delta}}}{(N^k(s,a)/L + 1) (1-\gamma)^2}}.
\end{align*}
Therefore, plugging into the previous display we obtain
\begin{align*}
\gamma\sum^K_{k=1} & \mathbb{E}_{s,a\sim d^{\pi^k}} \bs{\sqrt{\sum^L_{\ell=1} \br{\widehat{P}^k_\ell V^k (s,a) - \frac{1}{L}\sum^L_{\ell'=1} \widehat{P}^k_{\ell'} V^k(s,a)}^2}} \\&\leq  \sum^K_{k=1} \mathbb{E}_{s,a\sim d^{\pi^k}} \bs{\sqrt{\sum^L_{\ell=1} \widetilde{\mathcal{O}}\br{\frac{3\abs{\sspace}\log \br{\frac{1}{\delta}}}{(N^k(s,a)/L + 1) (1-\gamma)^2}} }} \\
&\leq  \sqrt{ K \sum^K_{k=1} \mathbb{E}_{s,a\sim d^{\pi^k}}\bs{\sum^L_{\ell=1} \widetilde{\mathcal{O}}\br{\frac{3\abs{\sspace}\log \br{\frac{1}{\delta}}}{(N^k(s,a)/L + 1) (1-\gamma)^2}} }} \\
&\leq  \widetilde{\mathcal{O}}\br{\sqrt{ K \sum^K_{k=1} \mathbb{E}_{s,a\sim d^{\pi^k}}\bs{\frac{3\abs{\sspace} L^2 \log \br{\frac{1}{\delta}}}{(N^k(s,a) + 1) (1-\gamma)^2}} }} 
\\
&\leq  \widetilde{\mathcal{O}}\br{\sqrt{ \frac{\abs{\sspace}^2 \log \br{\frac{1}{\delta}} K}{(1-\gamma)^2} \sum^K_{k=1} \mathbb{E}_{s,a\sim d^{\pi^k}}\bs{\frac{1}{N^k(s,a) + 1}} }} 
\end{align*}
Finally, we bound $\mathbb{E}_{s,a\sim d^{\pi^k}}\bs{\frac{1}{N^k(s,a) + 1}}$ using \Cref{lemma:count_based} under the event $L^k \leq L_{\max}:=\frac{\log(K/\delta)}{(1-\gamma)}$ which holds with probability $1-\delta$. Thanks to an union bound we have that with probability $1-2\delta$,
\begin{align*}
\gamma\sum^K_{k=1} \mathbb{E}_{s,a\sim d^{\pi^k}} &\bs{\sqrt{\sum^L_{\ell=1} \br{\widehat{P}^k_\ell V^k (s,a) - \frac{1}{L}\sum^L_{\ell'=1} \widehat{P}^k_{\ell'} V^k(s,a)}^2}} \\&\leq 
\widetilde{\mathcal{O}}\br{\sqrt{ \frac{\abs{\sspace}^2 \log \br{\frac{1}{\delta}} K}{(1-\gamma)^2} \abs{\sspace} \abs{\aspace} \log \br{K L_{\max}} + 4 \log \br{2KL_{\max}/\delta}  }} \\
&\leq \widetilde{\mathcal{O}}\br{\frac{  \sqrt{K \abs{\sspace}^2 \abs{\aspace} \log (1/\delta)}}{1-\gamma}}.
\end{align*}
Therefore, the proof is concluded also for the case of $Q$ value being updated as in \eqref{eq:update2}.
\end{proof}
\begin{lemma} \label{lemma:bounded_optimism_ensemble}
With probability $1-\delta$, it holds that for all $s,a\in \sspace\times \aspace$,
\begin{align*}
P V^k(s,a) - \widehat{P}^k_{\ell} V^k(s,a) &\leq \sqrt{\frac{\abs{\sspace}}{(N^k(s,a)/L + 1) (1-\gamma)^2}\log \br{\frac{\abs{\sspace}\abs{\aspace}  L K^2(K+1) }{(1-\gamma)\delta}}} + \frac{2}{K} \\&\phantom{=}+  \frac{2}{(N^k(s,a)/L + 1)(1-\gamma)} ~~~~~ \forall \ell,k \in [L]\times[K]
\end{align*}
In particular, the above statement implies that for all $k \in [K]$
\begin{align*}
P V^k(s,a) - \min_{\ell\in[L]}\widehat{P}^k_{\ell} V^k(s,a) &\leq \sqrt{\frac{\abs{\sspace}}{(N^k(s,a)/L + 1) (1-\gamma)^2}\log \br{\frac{\abs{\sspace}\abs{\aspace}  L K^2(K+1) }{(1-\gamma)\delta}}} + \frac{2}{K} \\&\phantom{=}+  \frac{2}{(N^k(s,a)/L + 1)(1-\gamma)} 
\end{align*}
\end{lemma}
\begin{proof}
    Let us introduce the value class of the possible value functions generated by \Cref{alg:theory_version}.
    i.e. $\mathcal{V} = \bc{f \in \mathbb{R}^{\abs{S}}~~|~~ \norm{f}_{\infty} \leq \frac{1}{1-\gamma}, f(s) \geq 0 ~~\forall s \in \sspace}$.
    Let us introduce a $\epsilon_{\mathrm{cov}}$-covering set $\mathcal{C}_{\epsilon_{\mathrm{cov}}}(\mathcal{V})$ such that for any $V \in \mathcal{V}$ there exists $\tilde{V}\in \mathcal{C}_{\epsilon_{\mathrm{cov}}}(\mathcal{V})$ such that $\norm{\tilde{V} - V}_{\infty} \leq \epsilon_{\mathrm{cov}}$.
Therefore let us denote by $\tilde{V}^k$ the element of $\mathcal{C}_{\epsilon_{\mathrm{cov}}}(\mathcal{V})$ such that $\norm{V^k - \tilde{V}^k }_{\infty} \leq \epsilon_{\mathrm{cov}} $.
Then, let us consider a generic $\tilde{V} \in \mathcal{C}_{\epsilon_{\mathrm{cov}}}(\mathcal{V})$,
\begin{align*}
P \tilde{V}(s,a) - \frac{1}{N^k_{\ell}(s,a)}\sum_{\bar{s},\bar{a},s' \in \mathcal{R}^k_\ell}\tilde{V}(s') \mathds{1}_{\bc{s,a=\bar{s},\bar{a}}} &= \frac{1}{N^k_{\ell}(s,a)} \sum_{\bar{s},\bar{a},s' \in \mathcal{R}^k_\ell} P \tilde{V} (s,a) \mathds{1}_{\bc{s,a=\bar{s},\bar{a}}} \\&\phantom{=}- \frac{1}{N^k_{\ell}(s,a)}\sum_{\bar{s},\bar{a},s' \in \mathcal{R}^k_\ell}\tilde{V}(s') \mathds{1}_{\bc{s,a=\bar{s},\bar{a}}} \\
&= \frac{1}{N^k_{\ell}(s,a)} \sum_{\bar{s},\bar{a},s' \in \mathcal{R}^k_\ell} \br{P \tilde{V} (s,a) - \tilde{V}(s')}\mathds{1}_{\bc{s,a=\bar{s},\bar{a}}} \\
\end{align*}
Then, notice that denoting $s'_n(s,a)$ the state sample after $s,a$ the $n^{th}$ time the state action pair was visited we have that
\[
\sum_{\bar{s},\bar{a},s' \in \mathcal{R}^k_\ell} \br{P \tilde{V} (s,a) - \tilde{V}(s')}\mathds{1}_{\bc{s,a=\bar{s},\bar{a}}} = \sum^{N^k_\ell(s,a)}_{n=1} \br{P \tilde{V} (s,a) - \tilde{V}(s'_n(s,a))}
\]
Applying directly the Azuma Hoeffding inequality is not possible because the number of elements in the sum, i.e. the number of visits $N^k_\ell(s,a)$ is not a random variable independent on the random variables $\bc{s'_n(s,a)}^{N^k_\ell(s,a)}_{n=1}$
(see \cite[Exercise 7.1]{lattimore2020bandit} ).

Therefore, we first apply the Azuma Hoeffding inequality for a specific $k$ and for a specific value of the visits $N^k_\ell(s,a)$. That is, it holds that with probability $1-\delta$
\begin{align*}
\sum_{\bar{s},\bar{a},s' \in \mathcal{R}^k_\ell} \br{P \tilde{V} (s,a) - \tilde{V}(s')}\mathds{1}_{\bc{s,a=\bar{s},\bar{a}}} &\leq \sqrt{\frac{N^k_{\ell}(s,a) \log (1/\delta)}{2 (1-\gamma)^2 }}
\end{align*}
Therefore via a union bound for $k \in [K]$ and $N^k_\ell(s,a) \in \bc{0,1, \dots, K}$ we have that with probability $1-\delta$ it holds that for all $k \in [K]$
\begin{align*}
\sum_{\bar{s},\bar{a},s' \in \mathcal{R}^k_\ell} \br{P \tilde{V} (s,a) - \tilde{V}(s')}\mathds{1}_{\bc{s,a=\bar{s},\bar{a}}} 
&\leq \sqrt{\frac{N^k_{\ell}(s,a) \log (K(K+1)/\delta)}{2(1-\gamma)^2 }} .
\end{align*}
Therefore, we can conclude that with probability at least $1-\delta$ for all $k\in[K]$
\begin{equation*}
P \tilde{V}(s,a) - \frac{1}{N^k_{\ell}(s,a)}\sum_{\bar{s},\bar{a},s' \in \mathcal{R}^k_\ell}\tilde{V}(s') \mathds{1}_{\bc{s,a=\bar{s},\bar{a}}} \leq \sqrt{\frac{\log (K(K+1)/\delta)}{2 N^k_{\ell}(s,a) (1-\gamma)^2 }}.
\end{equation*}
Now, by a another union bound over $\mathcal{C}_{\epsilon_{\mathrm{cov}}}(\mathcal{V})$ , $[K]$, $[L]$ and $\sspace\times\aspace$ and denoting $$\bar{P}^k_{\ell} \tilde{V}(s,a) := \frac{1}{N^k_{\ell}(s,a)}\sum_{\bar{s},\bar{a},s' \in \mathcal{R}^k_\ell}\tilde{V}(s') \mathds{1}_{\bc{s,a=\bar{s},\bar{a}}},$$  it holds that 
\begin{align*}
\mathbb{P}\bs{P \tilde{V}(s,a) - \bar{P}^k_{\ell} \tilde{V}(s,a) \leq \sqrt{\frac{\log (K(K+1)\abs{\sspace}\abs{\aspace}\abs{\mathcal{C}_{\epsilon_{\mathrm{cov}}}(\mathcal{V})} L/\delta)}{N^k_{\ell}(s,a) (1-\gamma)^2}} ~~\forall s,a,\ell,k\in\sspace\times\aspace\times [L]\times[K],~~\tilde{V}\in\mathcal{V}} \geq 1-\delta
\end{align*}
Therefore, now let us consider the element $\tilde{V}^k \in \mathcal{C}_{\epsilon_{\mathrm{cov}}}(\mathcal{V})$ such that $\norm{V^k - \tilde{V}^k}_{\infty} \leq \epsilon_{\mathrm{cov}}$.
Then, we have that for all $\ell \in [L]$
\begin{align*}
P V^k(s,a) - \bar{P}^k_{\ell} V^k(s,a) &= P \tilde{V}^k(s,a) - \bar{P}^k_{\ell} \tilde{V}^k(s,a) + (P - \bar{P}^k_\ell) (\tilde{V}^k - V^k) \\
&= P \tilde{V}^k(s,a) - \bar{P}^k_{\ell} \tilde{V}^k(s,a) + 2 \epsilon_{\mathrm{cov}} \\
&\leq \sqrt{\frac{\log (K(K+1)\abs{\sspace}\abs{\aspace}\abs{\mathcal{C}_{\epsilon_{\mathrm{cov}}}(\mathcal{V})} L/\delta)}{N^k_{\ell}(s,a) (1-\gamma)^2}} + 2 \epsilon_{\mathrm{cov}} \\
&\leq \sqrt{\frac{\abs{\sspace}}{N^k_{\ell}(s,a) (1-\gamma)^2}\log \br{\frac{K(K+1)\abs{\sspace}\abs{\aspace}  L}{(1-\gamma)\epsilon_{\mathrm{cov}}\delta}}} + 2 \epsilon_{\mathrm{cov}} \\
\end{align*}
With $\epsilon_{\mathrm{cov}}=K^{-1}$, we get
\begin{align*}
P V^k(s,a) - \bar{P}^k_{\ell} V^k(s,a) \leq \sqrt{\frac{\abs{\sspace}}{N^k_{\ell}(s,a) (1-\gamma)^2}\log \br{\frac{\abs{\sspace}\abs{\aspace}  L K^2(K+1) }{(1-\gamma)\delta}}} + \frac{2}{K} 
\end{align*}
Then, we can continue as follows 
\begin{align*}
\widehat{P}^k_\ell V^k(s,a) &= \frac{N^k_\ell(s,a)}{N^k_\ell(s,a) + 2} \bar{P}^k_\ell V^k(s,a) \\
&\geq \frac{N^k_\ell(s,a)}{N^k_\ell(s,a) + 2} 
\bs{P V^k(s,a) - \sqrt{\frac{\abs{\sspace}}{N^k_{\ell}(s,a) (1-\gamma)^2}\log \br{\frac{\abs{\sspace}\abs{\aspace}  L K^2(K+1) }{(1-\gamma)\delta}}} - \frac{2}{K} } \\
&=
P V^k(s,a) - \frac{2}{N^k_\ell(s,a) + 2} 
P V^k(s,a) - \sqrt{\frac{\abs{\sspace}}{(N^k_{\ell}(s,a) + 2) (1-\gamma)^2}\log \br{\frac{\abs{\sspace}\abs{\aspace}  L K^2(K+1) }{(1-\gamma)\delta}}} - \frac{2}{K}  \\
&\geq 
P V^k(s,a) - \frac{2}{(N^k_\ell(s,a) + 2)(1-\gamma)} 
 - \sqrt{\frac{\abs{\sspace}}{(N^k_{\ell}(s,a) + 2) (1-\gamma)^2}\log \br{\frac{\abs{\sspace}\abs{\aspace}  L K^2(K+1) }{(1-\gamma)\delta}}} - \frac{2}{K}  
\end{align*}
Finally, rearranging and using that $N^k_\ell (s,a) = \floor{\frac{N^k(s,a)}{L}} \geq \frac{N^k(s,a)}{L} - 1, $
we obtain that with probability $1-\delta$, it holds that for all $\ell\in[L]$, $k \in[K]$, $s,a\in\sspace\times\aspace$
\begin{align*}
P V^k(s,a) - \widehat{P}^k_{\ell} V^k(s,a) &\leq \sqrt{\frac{\abs{\sspace}}{(N^k(s,a)/L + 1) (1-\gamma)^2}\log \br{\frac{\abs{\sspace}\abs{\aspace}  L K^2(K+1) }{(1-\gamma)\delta}}} + \frac{2}{K} \\&\phantom{=}+  \frac{2}{(N^k(s,a)/L + 1)(1-\gamma)} 
\end{align*}
\end{proof}
\subsection{Upper bound the regret of the reward player}
\thmcostregret*
\begin{proof}
We decompose the regret as follows 
\begin{align*}
\sum^K_{k=1} \innerprod{c_{\mathrm{true}} - c^k}{d^{\pi^k} - d^\expert} &= \sum^K_{k=1} \innerprod{c_{\mathrm{true}} - c^k}{\mathbf{e}_{s^k_{L^k}} - \widehat{d^\expert}} \\
&\phantom{=}+\sum^K_{k=1} \innerprod{c_{\mathrm{true}} - c^k}{ d^{\pi^k}-\mathbf{e}_{s^k_{L^k}}} \\
&\phantom{=}+\sum^K_{k=1} \innerprod{c_{\mathrm{true}} - c^k}{ \widehat{d^\expert} - d^\expert}
\end{align*}
For the first term, we can invoke a standard online gradient descent bound
and get
\begin{align*}
    \sum^K_{k=1} \innerprod{c_{\mathrm{true}} - c^k}{\mathbf{e}_{s^k_{L^k}} - \widehat{d^\expert}} &\leq \frac{2}{\eta} + \eta K \norm{\widehat{d^\expert}}_2/2 \\
    &\leq \frac{2}{\eta} + \eta K \norm{\widehat{d^\expert}}_1/2 \\
    &\leq \frac{2}{\eta} + \frac{\eta K}{2}\\
\end{align*}
Therefore choosing $\eta = \sqrt{\frac{4}{K}}$, we get
\begin{align*}
    \sum^K_{k=1} \innerprod{c_{\mathrm{true}} - c^k}{\mathbf{e}_{s^k_{L^k}} - \widehat{d^\expert}} &\leq 2\sqrt{K}.
\end{align*}
Then, we can handle the remaining two terms. In particular $\sum^K_{k=1} \innerprod{c_{\mathrm{true}} - c^k}{ d^{\pi^k}-\mathbf{e}_{s^k_{L^k}}}$ is the sum of a martingale difference sequence.
Therefore, applying the Azuma-Hoeffding inequality it holds that with probability $1-\delta$
\begin{equation*}
\sum^K_{k=1} \innerprod{c_{\mathrm{true}} - c^k}{ d^{\pi^k}-\mathbf{e}_{s^k_{L^k}}} \leq \sqrt{2 K \log (1/\delta)}
\end{equation*}
where we used that $\abs{\innerprod{c_{\mathrm{true}} - c^k}{ d^{\pi^k}-\mathbf{e}_{s^k_{L^k}}}} \leq 2$ for all $k \in [K]$.
Finally for the expert concentration term, we have that
\begin{align*}
    \sum^K_{k=1} \innerprod{c_{\mathrm{true}} - c^k}{ \widehat{d^\expert} - d^\expert} \leq K \sqrt{\abs{\sspace}\abs{\aspace}} \norm{d^{\expert} - d^\expert}_{\infty}
\end{align*}
Then, for any fixed state action pair $s,a$ with probability $1- \delta/(\abs{\sspace}\abs{\aspace})$ by Azuma-Hoeffding inequality it holds that
\begin{align*}
d^{\expert}(s) - d^\expert(s) = \frac{1}{\abs{\mathcal{D}_\expert}} \sum_{s' \in \mathcal{D}_\expert } \mathds{1}_{\bc{s'=s}} - d^\expert (s) \leq \sqrt{\frac{\log (\abs{\sspace}\abs{\aspace}\delta^{-1})}{2 \abs{\mathcal{D}_\expert}}}
\end{align*}
Therefore, by a union bound it holds that with probability $ 1-\delta$,
\begin{equation*}
\norm{d^{\expert} - d^\expert}_{\infty} \leq \sqrt{\frac{\log (\abs{\sspace} \abs{\aspace}\delta^{-1})}{2 \abs{\mathcal{D}_\expert}}}.
\end{equation*}
Putting together, the bounds on the three terms allow to conclude the proof.
\end{proof}
\section{Technical Lemmas}
\lemmapdl*
\begin{proof}
    \begin{equation*}
        \innerprod{d^{\pi'}}{\hat{Q}} = \innerprod{d^{\pi'}}{\hat{Q} - \cost - \gamma P \hat{V}^\pi} + \innerprod{d^{\pi'}}{\cost + \gamma P \hat{V}^\pi}
    \end{equation*}
Then, using the property of occupancy measure we have that $\innerprod{d^{\pi'}}{\cost} = (1 - \gamma)\innerprod{\initial}{V^{\pi'}}$ where $V^{\pi'}$ is the value function of the policy $\pi'$ in the MDP.
Then, it holds that
\begin{align*}
        \innerprod{d^{\pi'}}{\hat{Q}} &= \innerprod{d^{\pi'}}{\hat{Q} - \cost - \gamma P \hat{V}^\pi} + (1 - \gamma)\innerprod{\initial}{V^{\pi'}}+\innerprod{d^{\pi'}}{\gamma P \hat{V}^\pi} \\
        & = \innerprod{d^{\pi'}}{\hat{Q} - \cost - \gamma P \hat{V}^\pi} + (1 - \gamma)\innerprod{\initial}{V^{\pi'}}+\innerprod{\gamma P^T d^{\pi'}}{\hat{V}^\pi}
        \\ & = \innerprod{d^{\pi'}}{\hat{Q} - \cost - \gamma P \hat{V}^\pi} + (1 - \gamma)\innerprod{\initial}{V^{\pi'}}+\innerprod{E^T d^{\pi'} -(1 - \gamma) \initial}{\hat{V}^\pi} \\
        & = \innerprod{d^{\pi'}}{\hat{Q} - \cost - \gamma P \hat{V}^\pi} + (1 - \gamma)\innerprod{\initial}{V^{\pi'} - \hat{V}^\pi}+\innerprod{E^T d^{\pi'}}{\hat{V}^\pi}.
    \end{align*}
    Rearranging and using the definition of $\widehat{V}^\pi$ yields the conclusion.
\end{proof}
\begin{lemma}
\label{lemma:count_based}
Let us assume that $L^k \leq L_{\max}$ for all $k \in [K]$. Then, invoking \cite[Lemma D.4]{cohen2019learning}, it holds that with probability $1-\delta$
\begin{equation*}
\sum^K_{k=1}\mathbb{E}_{s,a\sim d^{\pi^k}} \bs{\frac{1}{N^k(s,a)/L + 1}}  \leq 
2 L \abs{\sspace} \abs{\aspace} \log \br{K L_{\max}} + 4 \log \br{2KL_{\max}/\delta} 
\end{equation*}
\end{lemma}
\begin{proof}
Let us assume that $L^k \leq L_{\max}$ for all $k \in [K]$. It holds that with probability $1-\delta$
\begin{align*}
\sum^K_{k=1}\mathbb{E}_{s,a\sim d^{\pi^k}} \bs{\frac{1}{N^k(s,a)/L + 1}} &\leq 2\sum^K_{k=1}\sum^{L^k}_{t=1}\frac{1}{N^k(s^k_t,a^k_t)/L + 1} + 4 \log \br{2KL_{\max}/\delta} \\
& \leq 2\sum_{s,a\in\sspace\times\aspace}\sum^K_{k=1}\sum^{L^k}_{t=1}\frac{\mathds{1}_{\bc{s,a=s^k_t,a^k_t}}}{N^k(s,a)/L + 1} + 4 \log \br{2KL_{\max}/\delta} \\
& \leq 2 L \sum_{s,a\in\sspace\times\aspace}\sum^K_{k=1}\sum^{L^k}_{t=1}\frac{\mathds{1}_{\bc{s,a=s^k_t,a^k_t}}}{N^k(s,a)+ 1} + 4 \log \br{2KL_{\max}/\delta} \\
& \leq 2 L \sum_{s,a\in\sspace\times\aspace}\sum^K_{k=1}\sum^{L^k}_{t=1}\frac{\mathds{1}_{\bc{s,a=s^k_t,a^k_t}}}{\sum^k_{\tau=1}\sum^{L_\tau}_{t=1} \mathds{1}_{\bc{s,a=s^\tau_t,a^\tau_t}} + 1} + 4 \log \br{2KL_{\max}/\delta} \\
& \leq 2 L \sum_{s,a\in\sspace\times\aspace} \log \br{\sum^K_{k=1}\sum^{L^k}_{t=1}\mathds{1}_{\bc{s,a=s^k_t,a^k_t}}} + 4 \log \br{2KL_{\max}/\delta} \\
& \leq 2 L \abs{\sspace} \abs{\aspace} \log \br{K L_{\max}} + 4 \log \br{2KL_{\max}/\delta} \\
\end{align*}
where we used \Cref{lemma:numerical_sequence} for $f(x) = x^{-1}$.
\end{proof}
\begin{lemma}
\label{lemma:numerical_sequence}
Let $a_0 \geq 0$ and $f [0, \infty ) \rightarrow [0, \infty ) $ be a non increasing function , then
\begin{equation*}
\sum^T_{t=1} \alpha_t f(a_0 + \sum^T_{t=1} \alpha_t) \leq \int^{\sum^T_{t=1} a_t}_{a_0} f(x) dx
\end{equation*}
\end{lemma}
\begin{proof}
See \cite{orabona2023modern} Lemma 4.13.
\end{proof}
\begin{lemma}
    \label{lemma:slow} The sequence of policies $\bc{\pi^k}^K_{k=1}$ generated by \Cref{alg:theory_version}
    and let $d^\pi$ denote the occupancy measure for the policy $\pi$. Then it holds that
    \begin{equation*}
    \forall k ~~\in [K]~~~\norm{d^{\pi^k} - d^{\pi^{k+1}}}_1 \leq \frac{\eta Q_{\max}}{(1-\gamma)}
    \end{equation*}
    \end{lemma}
    \begin{proof}
        By Lemma A.1 in \cite{sun2019dual} it holds that
        \begin{equation*}
    \norm{d^{\pi^k} - d^{\pi^{k+1}}}_1 \leq \frac{1}{1-\gamma} \mathbb{E}_{x \sim d^{\pi^k}}\bs{\norm{\pi^k(\cdot|x) - \pi^{k+1}(\cdot|x)}_1}
        \end{equation*}
    Then, we notice that by $1$-strong convexity of the KL divergence it holds that
        \begin{align*}
     \frac{1}{2}\mathbb{E}_{x \sim d^{\pi^k}}&\bs{\norm{\pi^k(\cdot|x) - \pi^{k+1}(\cdot|x)}^2_1} \leq  \frac{1}{2}\mathbb{E}_{x \sim d^{\pi^k}}\bs{D_{KL}(\pi^{k+1}(\cdot|x), \pi^k(\cdot|x))} \\
     &\leq  \frac{1}{2}\mathbb{E}_{x \sim d^{\pi^k}}\sum_{a \in \aspace} \pi^{k+1}(a|x) \br{- \eta Q_k(x,a) - \log\br{\sum_{a \in \aspace} \pi^{k}(a|x) \exp(- \eta Q_k(x,a)) }} \\
     &= -\frac{\eta}{2} \mathbb{E}_{x \sim d^{\pi^k}}\sum_{a \in \aspace} \pi^{k+1}(a|x) Q_k(x,a) - \frac{1}{2}\mathbb{E}_{x \sim d^{\pi^k}} \log\br{\sum_{a \in \aspace} \pi^{k}(a|x) \exp(- \eta Q_k(x,a)) } \\
     & \leq -\frac{\eta}{2} \mathbb{E}_{x \sim d^{\pi^k}}\sum_{a \in \aspace} \pi^{k+1}(a|x) Q_k(x,a)  + \frac{\eta }{2}\mathbb{E}_{x \sim d^{\pi^k}}\sum_{a \in \aspace} \pi^{k}(a|x) Q_k(x,a)
        \end{align*}
        where the last inequality follows by Jensen's inequality and convexity of $- \log$.
        Hence, we continue the upper bound as follows
        \begin{align*}
    \frac{1}{2}\mathbb{E}_{x \sim d^{\pi^k}}\bs{\norm{\pi^k(\cdot|x) - \pi^{k+1}(\cdot|x)}^2_1} &= \frac{\eta}{2}\mathbb{E}_{x \sim d^{\pi^k}}\sum_{a \in \aspace} Q_k(x,a) \cdot (\pi^k(\cdot|x) - \pi^{k+1}(\cdot|x)) \\
    & \leq \frac{\eta Q_{\max}}{2}\cdot \mathbb{E}_{x \sim d^{\pi^k}}\bs{\norm{\pi^k(\cdot|x) - \pi^{k+1}(\cdot|x)}_1}
        \end{align*}
        Which implies, by Jensen's inequality and diving both sides by $\frac{1}{2}\mathbb{E}_{x \sim d^{\pi^k}}\bs{\norm{\pi^k(\cdot|x) - \pi^{k+1}(\cdot|x)}_1}$ that
        \begin{equation*}
            \mathbb{E}_{x \sim d^{\pi^k}}\bs{\norm{\pi^k(\cdot|x) - \pi^{k+1}(\cdot|x)}_1} \leq \eta Q_{\max}.
        \end{equation*}
    \end{proof}
\begin{lemma}\textbf{Samuelson's inequality}
Let us consider $L$ scalars $\bc{X_{\ell}}^L_{\ell=1}$ and denote the sample mean as $\bar{X} = L^{-1}\sum^L_{\ell=1} X_\ell$
and the empirical standard deviation as $\hat{\sigma} = \sqrt{\frac{\sum^L_{\ell=1}(X_\ell - \bar{X})^2}{L-1}}$, then it holds that
\begin{equation*}
\bar{X} - \sqrt{L-1} \hat{\sigma} \leq X_{\ell} \leq \bar{X} + \sqrt{L-1} \hat{\sigma} ~~~ \forall ~~~\ell\in[L]
\end{equation*}
\label{lemma:samuelson}
\end{lemma}
\begin{proof}
Let us consider an arbitrary vector $v \in \mathbb{R}^L$. Then, we have that $\norm{v}_\infty \leq \norm{v}_2$. At this point let us consider $v = [X_1 - \bar{X}, \dots, X_L - \bar{X}]^T$. Moreover,
let us define as $\ell^\star$ the index such that $\norm{v}_{\infty} = \abs{X_{\ell^\star} - \bar{X}}$.
Then, we have that for all $\ell \in [L]$,
\begin{equation*}
\abs{X_{\ell} - \bar{X}} \leq \abs{X_{\ell^\star} - \bar{X}}
\leq \sqrt{\sum^L_{\ell=1} (X_\ell - \bar{X})^2} = \sqrt{L-1}\hat{\sigma}.
\end{equation*}
Therefore, rewriting the absolute value it holds that
\begin{equation*}
\bar{X} - \sqrt{L-1}\hat{\sigma} \leq X_\ell \leq \bar{X} + \sqrt{L-1}\hat{\sigma}
\end{equation*}
\end{proof}
The next lemma says that the effective horizon in the original MDP and the binarized MDP is equal up to a $\log_2\br{\abs{\sspace}}$ factor.
\begin{lemma} \label{lemma:eff_horizon}
It holds that \begin{equation*}
\frac{1}{1 - \gamma^{1/\log_2 \abs{\sspace}}} \leq \frac{\log_2{\abs{\sspace}} + 2}{1-\gamma}.
\end{equation*}
\end{lemma}
\begin{proof}
\begin{align*}
\frac{1}{1-\gamma^{1/\log_2 \abs{\sspace}}} &= \frac{1}{1-\gamma} \frac{1-\gamma}{1 -\gamma^{1/\log_2 \abs{\sspace}}} \\
&= \frac{1}{1-\gamma} \frac{1-\gamma^{\log_2\abs{\sspace}}_{\mathrm{bin}}}{1 -\gamma_{\mathrm{bin}}} \\
&= \frac{1}{1-\gamma} \sum^{\log_2\abs{\sspace}+1}_{t=0} \gamma^t_{\mathrm{bin}} \\
&\leq \frac{\log_2 \abs{\sspace}+2}{1-\gamma}.
\end{align*}
\end{proof}
\section{Implementation details}
\label{sec:training_proc}

\textbf{Environment:} We use the Hopper-v5, Ant-v5, HalfCheetah-v5, and Walker2d-v5 environments from OpenAI Gym.

\textbf{Expert Samples:} The expert policy is trained using SAC. The training configuration uses 3000 epochs. The agent explores randomly for the first 10 episodes before starting policy learning. A replay buffer of 1 million experiences is used, with a batch size of 100 and a learning rate of 1e-3. The temperature parameter ($\alpha$) is set to 0.2. The policy updates occur every 50 steps, with 1 update per interval. After training 64 experts trajectories are collected to be used later for the agent training. 

\textbf{IL algorithms implementation:} Our starting code base is taken from the repository of \texttt{f-IRL}\footnote{\url{https://github.com/twni2016/f-IRL/tree/main}} \cite{ni2021f}, and the implementation of the other algorithms are based on this one. For more details about the implementation please refer to our repository. 
The most important hyperparameters are reported in Table \ref{tab:hyperparameters}

\begin{itemize}
   \item \textbf{ML-IRL, f-IRL and rkl:} These algorithms were already implemented in the \texttt{f-IRL} repository. The method leverages SAC as the underlying reinforcement learning algorithm and different type of objectives for the cost update. The multi-Q-network exploration bonus is implemented inside the SAC update, there we keep track of multiple Q-networks and use their mean and standard deviation to update the policy. The clipping is applied on the standard deviation which serves as the exploration bonus. 
   \item \textbf{CSIL:} We started from the \texttt{f-IRL} implementation, maintaining the same hyperparameters for a fair comparison. The key modification was removing reward model training from the RL loop, instead training it only once before entering the loop using behavioral cloning and $L_2$ normalization, after which the reward model remained fixed throughout the training.
   \item \textbf{OPT-AIL (state-only and state-action):} We started from the implementation of \texttt{ML-IRL} and added the OPT-AIL exploration bonus, incorporating optimism-regularized Bellman error minimization for Q-value functions as described in the original article \cite{xu2024provably}.
    
    The updated Q-loss can be formulated as:
    
    \begin{equation}
    \mathcal{L}_Q = \mathbb{E}\left[\left(Q_\theta(s,a) - \left(r + \gamma(1-d)(Q_{\bar{\theta}}(s',a') - \alpha \log \pi(a'|s'))\right)\right)^2\right] - \lambda \mathbb{E}[Q_\theta(s,a)]
    \end{equation}
    
    Where:
    - $Q_\theta$ is the current Q-network
    - $Q_{\bar{\theta}}$ is the target Q-network
    - $r$ is the learned reward
    - $\gamma$ is the discount factor
    - $d$ is the done flag
    - $\alpha$ is the entropy coefficient
    - $\lambda$ is the optimism regularization parameter
    - The expectation is taken across the data distribution $\mathcal{D}$ sampled from the replay buffer, which includes state-action-reward-next state-done transitions and individual state-action pairs.

   \item \textbf{SQIL} \cite{sqil}: was implemented by initializing a replay buffer with expert trajectories and assigning them a reward of 1, while collecting additional on-policy trajectories from the agent's current policy with a reward of 0. During training, the \texttt{SAC} agent learns from both expert and agent-generated transitions, effectively learning to imitate expert behavior through the asymmetric reward structure. The agent updates its policy by sampling from this mixed replay buffer, where the expert transitions provide a high-reward signal to guide the learning process.
   
   \item \textbf{GAILs:} we used the implementation available from \texttt{Stable-Baselines3} \cite{stable-baselines3}. This is the only method not based on \texttt{SAC}.
\end{itemize}

\begin{table}[]
\centering
\caption{Core Hyperparameters Across Environments}
\begin{tabular}{lcccc}
\toprule
Parameter & Walker2d & Humanoid & Hopper & Ant \\
\midrule
\text{Number of Iterations} & 1.5\,\text{M} & 1\,\text{M} & 1\,\text{M} & 1.2\,\text{M} \\
Reward Network size & [64, 64] & [64, 64] & [64, 64] & [128, 128] \\
Policy Network size & [256, 256] & [256, 256] & [256, 256] & [256, 256] \\
Reward Learning Rate & 1e-4 & 1e-4 & 1e-4 & 1e-4 \\
SAC Learning Rate & 1e-3 & 1e-3 & 1e-3 & 1e-3 \\
\bottomrule
 \label{tab:hyperparameters}
\end{tabular}
\end{table}

\subsection{Hyperparameters tuning}

To determine how different numbers of neural networks changed the performance, we conducted an ablation study on both the clipping value and the number of neural networks. We performed this analysis for the ML-IRL algorithm on the Ant-V5 environment. We chose this environment since previous experiments showed that its higher complexity led to higher variance in the performance of different algorithms. It was necessary to perform a grid search on the number of neural networks because we noticed that different numbers of neural networks preferred different clipping values.

The results are reported in \Cref{fig:best_clip_nn}. As we observed, in all cases, adding more neural networks leads to better performances. However, this improvement does not increase proportionally with the number of neural networks; in fact, the run with 10 neural networks is outperformed by the one with 4. This led us to select 4 as the fixed value of neural networks, also justified by the much slower training time of the 10-network case.

\begin{figure}[]
\centering
\includegraphics[width=0.6\linewidth]{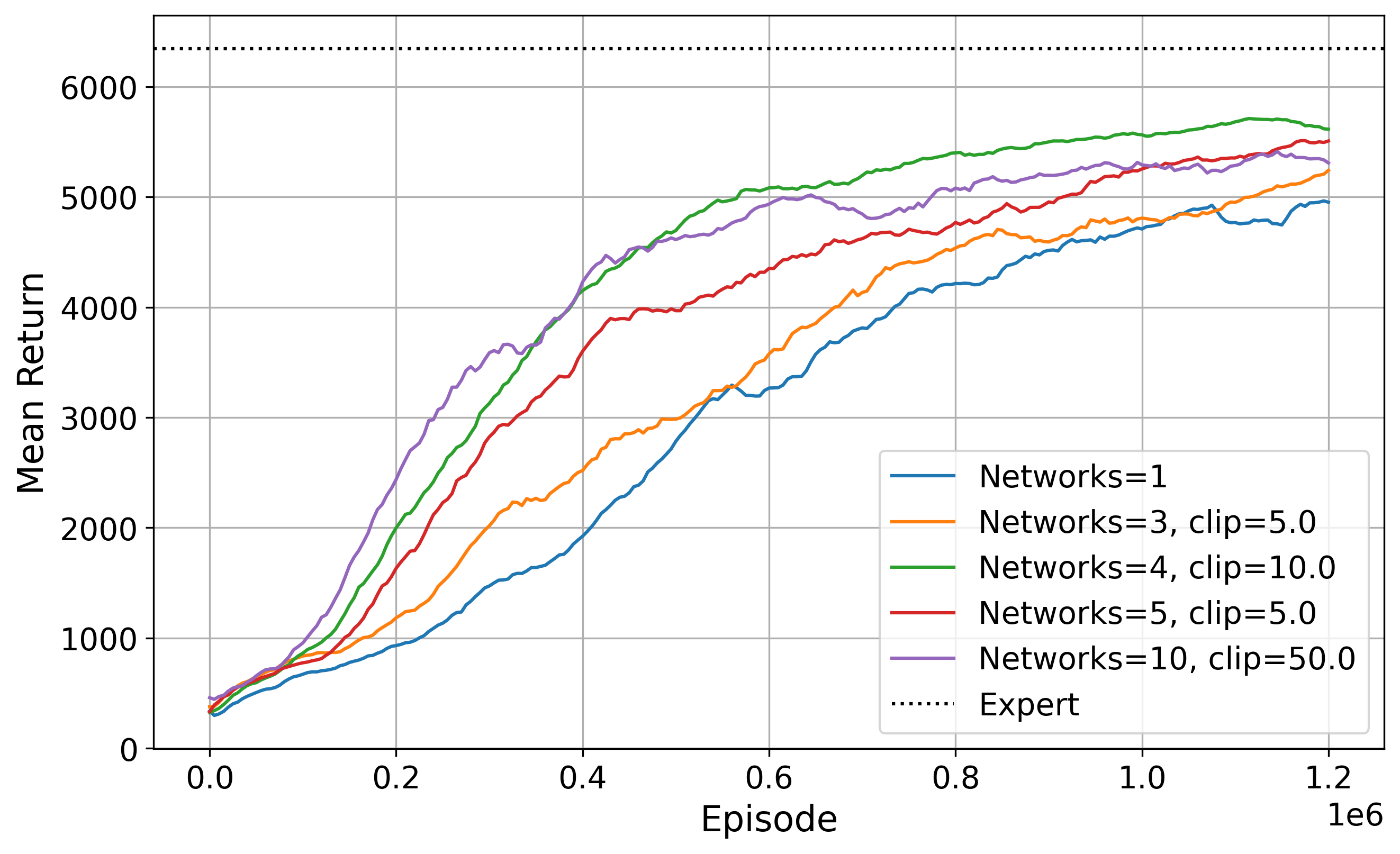}
\caption{Mean return of ML-IRL Ant-v5 with a different number of neural networks. The grid search for the clipping values was performed over the following values [0.1 0.5 1 5 10 50]. Results are averaged over 3 seeds. 
\label{fig:best_clip_nn}}
\end{figure}

For every environment and algorithm, we performed a grid search over different clipping values. The range of clipping values varied across algorithms. Figure \ref{fig:comparison_plots} shows the different values used in the search and their impact on performance. The difference in performance across clipping values is small in simpler environments (e.g., Hopper or Walker2d) while it becomes more evident in more complex environments with larger state-action spaces. These plots also show the necessity of the clipping for the exploration bonus. In most environments and algorithms, when a large clipping value is applied, it leads to performance degradation. 

Empirically, the Q-network's standard deviation diverges due to unclipped Q-values. Without value clipping, high Q-values for specific state-action pairs increase the probability of being visited, causing more of these pairs to accumulate in the replay buffer across different rollouts and potentially amplifying the standard deviation across the q-network for the next update. This justifies the necessity of a clipping value on the exploration bonus.

\begin{figure}[]
    \centering
    \begin{subfigure}[b]{0.48\textwidth}
        \includegraphics[width=\textwidth]{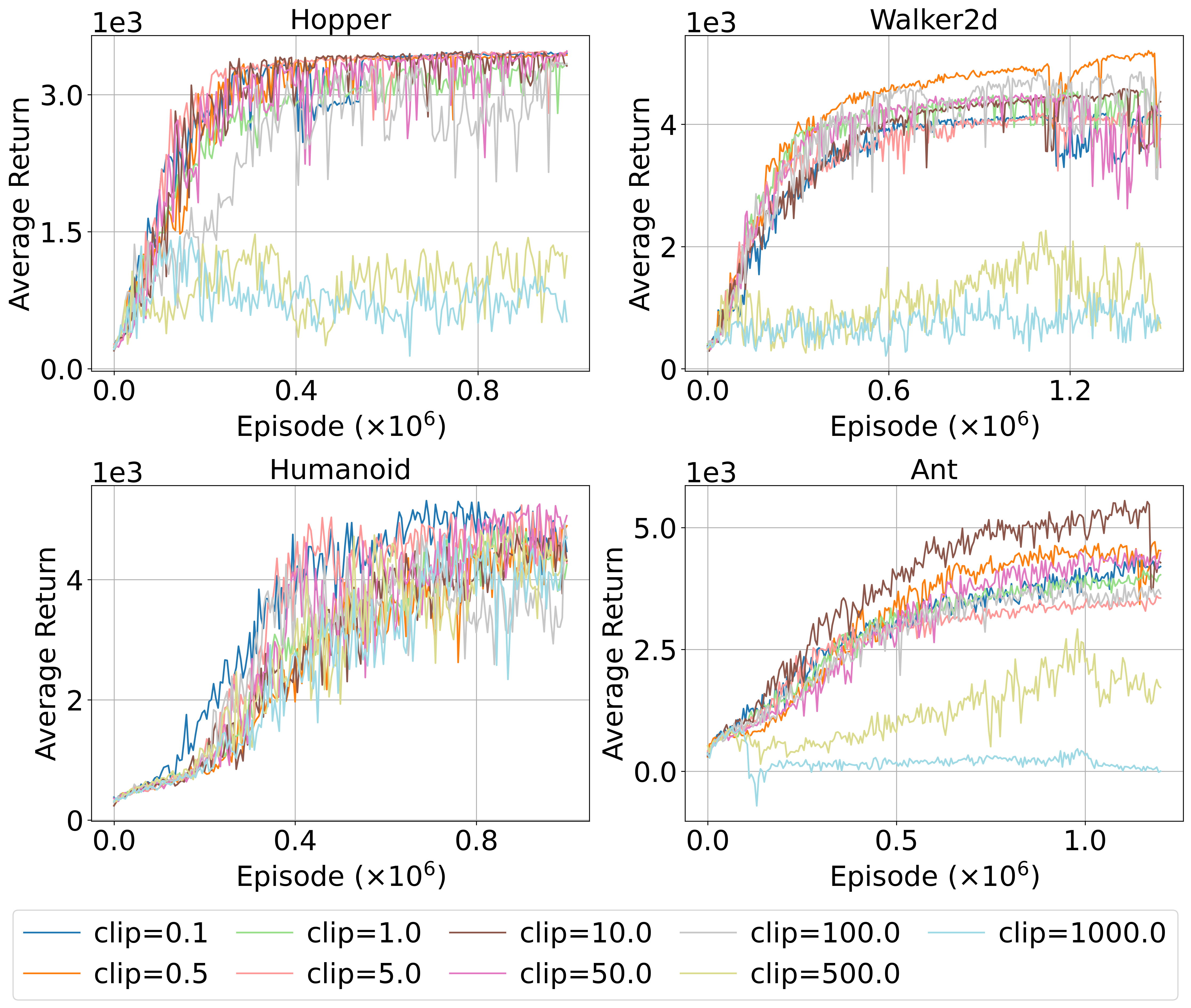}
        \caption{CSIL Comparison}
    \end{subfigure}
    \hfill
    \begin{subfigure}[b]{0.48\textwidth}
        \includegraphics[width=\textwidth]{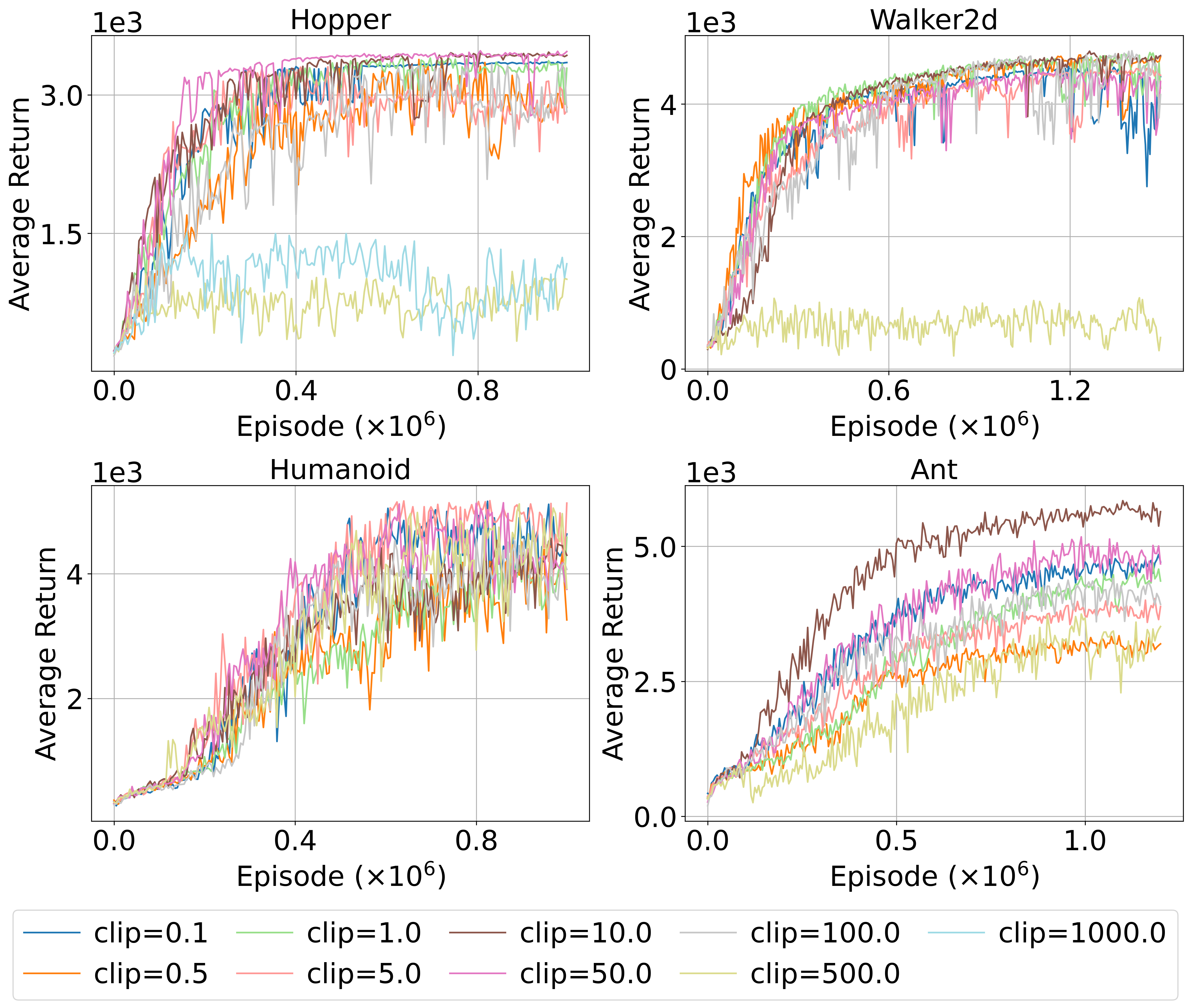}
        \caption{ML-IRL Comparison}
    \end{subfigure}

    \vspace{1em}

    \begin{subfigure}[b]{0.48\textwidth}
        \includegraphics[width=\textwidth]{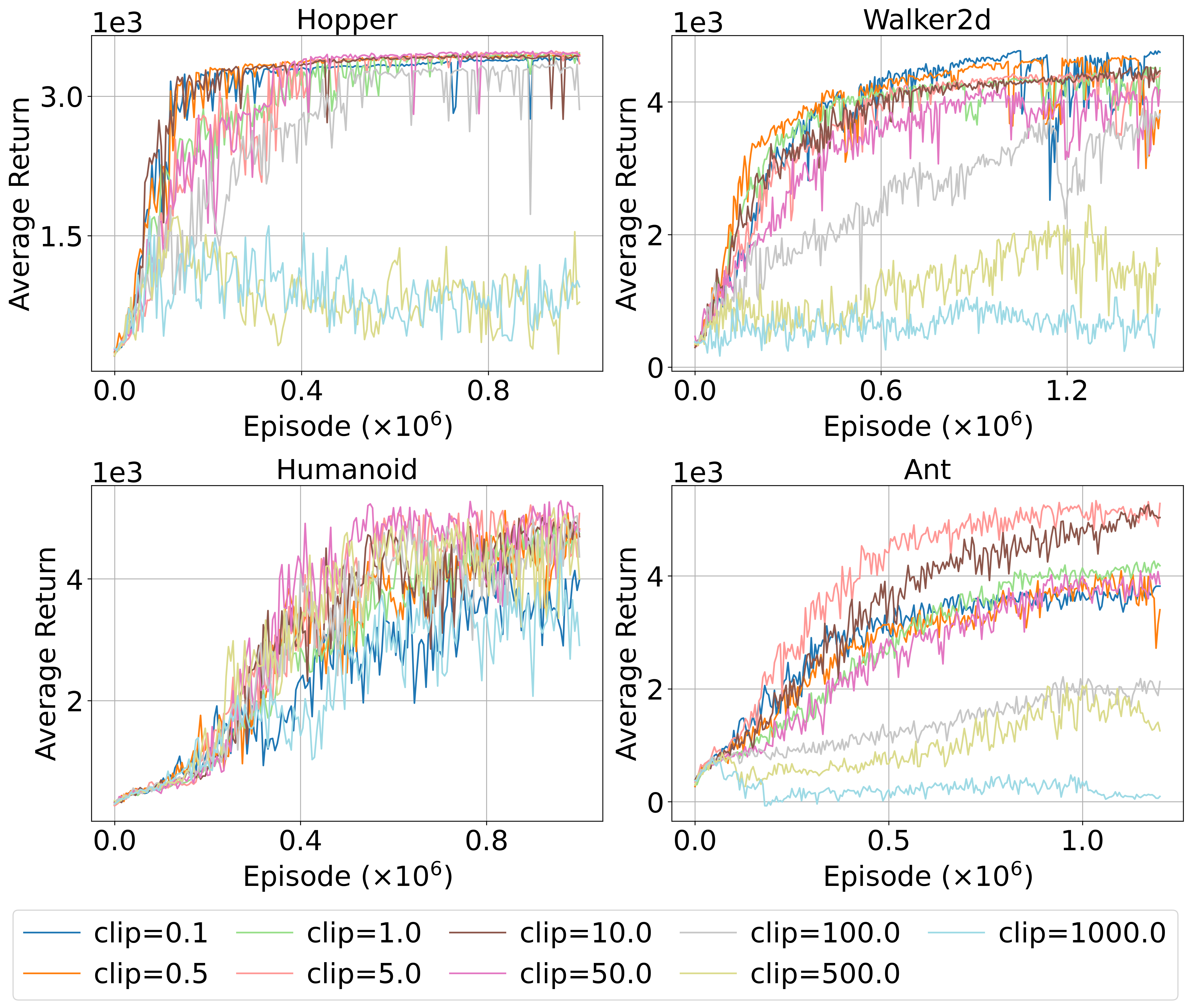}
        \caption{ML-IRL-SA Comparison}
    \end{subfigure}
    \hfill
    \begin{subfigure}[b]{0.48\textwidth}
        \includegraphics[width=\textwidth]{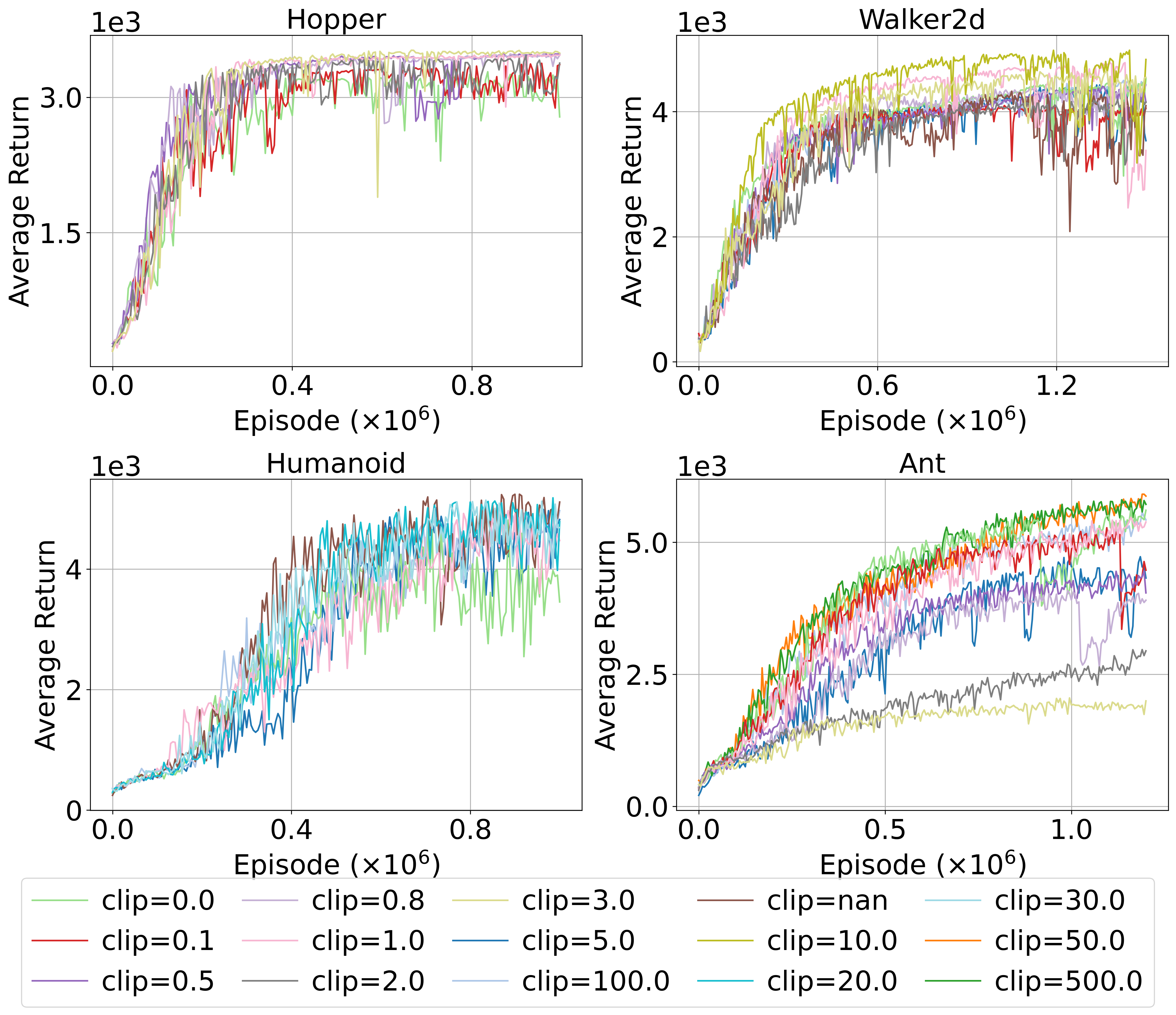}
        \caption{RKL Comparison}
    \end{subfigure}
    
    \caption{Comparison of clipping values across different environments, showing the effect on average return for each environment.}
    \label{fig:comparison_plots}
\end{figure}

\newpage
\section{Experiments with single expert trajectory}

Here we report the we report the result of the experiments using a single trajectory. Our findings indicate that the performance remained consistent regardless of the number of trajectories used and the performance are comparable to the ones with 16 trajectories. Notable differences in performance improvement were observed in Humanoid-v5 and ant state environments in the state only settings, where a more pronounced gap was evident.

\begin{figure*}[h]
    \centering
\includegraphics[width=\textwidth]{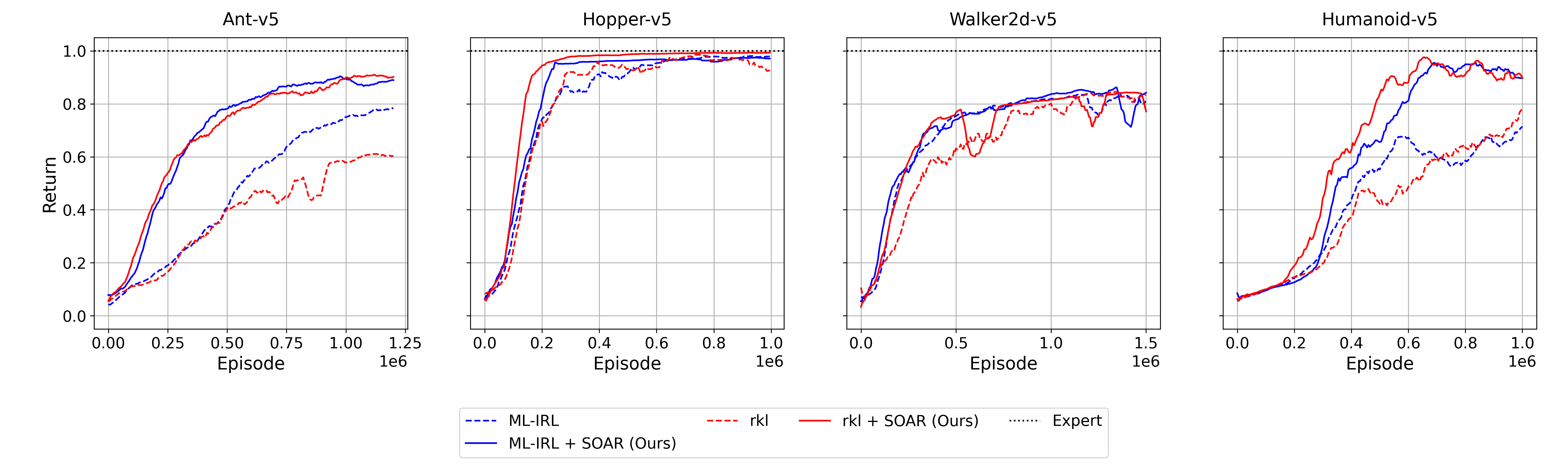}
    \caption{
    \small{\textbf{Experiments from State-Only Expert Trajectories}. 1 expert trajectories, average over 3 seeds, $L=4$ 
     Clipping values $\sigma$ - ML-IRL: [Ant: 0.1, Hopper: 0.1, Walker2d: 50.0, Humanoid: 0.5],  
    rkl: [Ant: 0.1, Hopper: 0.5, Walker2d: 1.0, Humanoid: 50.0]}}
    \label{fig:state_only_1traj}
\end{figure*}

\begin{figure*}[h]
    \centering
\includegraphics[width=\textwidth]{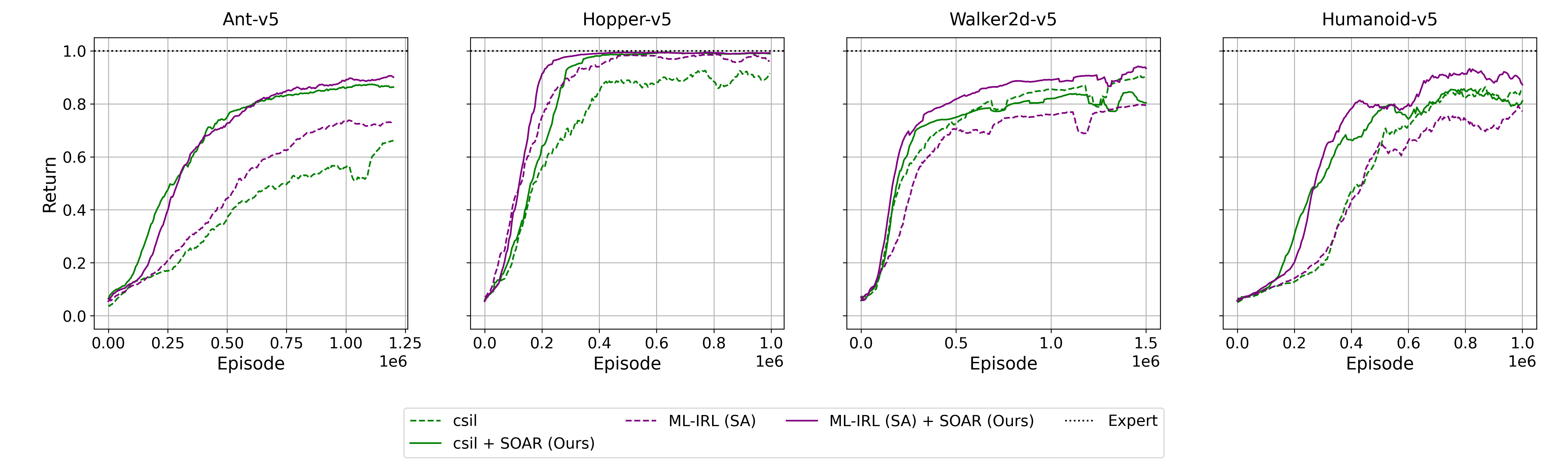}
    \caption{
    \small{\textbf{Experiments from State-Action Expert Trajectories}. 1 expert trajectories, average over 3 seeds, $L=4$.
    Clipping values $\sigma$ - CSIL: [Ant: 0.5, Hopper: 50.0, Walker2d: 0.1, Humanoid: 0.1],  
ML-IRL(SA): [Ant: 0.1, Hopper: 0.5, Walker2d: 0.1, Humanoid: 1.0]}}
\label{fig:state_actions_1traj}
\end{figure*}

\section{Omitted Pseudocodes}
\label{app:pseudo}
This section introduces the omitted pseudocodes to clarify the implementation of the algorithms based on SOAR. We first give a pseudocode (see \Cref{alg:deepmeta}) that mirrors \Cref{alg:meta} in the setting where deep neural network approximation is needed due to the continuous structure of the state-action space.
The critic training is the same as in the standard SAC \cite{Haarnoja:2018} but we report it in \Cref{alg:updateCritics} for safe completeness. 
Notice that we adopt the double critic training originally proposed in \cite{vanhasselt2015deepreinforcementlearningdouble} to avoid an excessive underestimation of the critics value\footnote{Notice that \cite{vanhasselt2015deepreinforcementlearningdouble} talks about excessive overestimation of the prediction target in the critic training rather than underestimation. This difference is due to the fact that their paper casts RL as reward maximization while we adopt a cost minimization perspective. For the same reason we take the maximum between the two critics rather than the minimum as done in \cite{vanhasselt2015deepreinforcementlearningdouble}.}.
\begin{algorithm}[H]
\caption{Base Method + SOAR pseudocode \label{alg:deepmeta}}
\begin{algorithmic}[1]
\REQUIRE Policy step size $\eta$, cost step size $\alpha$, expert dataset $\mathcal{D}_{\tau_E}$, discount factor $\gamma$, maximum standard deviation parameter $\sigma$, 
\State Initialize actor network $\pi_{\psi^1}$ randomly.
\State Initialize the cost network $c_{w^1}$ randomly.
\STATE Initialize the $L$ critics $\{Q_{\theta^1_1}, \ldots, Q_{\theta^1_L}\}$ randomly.
\STATE Initialize the $L$ target critics $\{Q_{\theta^{1,\text{targ}}_1}, \ldots, Q_{\theta^{1,\text{targ}}_L}\}$ randomly.
\State Initialize $L$ empty replay buffers $\bc{\mathcal{D}^k_\ell}^L_{\ell=1}$. (One for each critic)
\For{$k = 1$ to $K$}
    \State $\tau_\ell^k \gets \textsc{CollectTrajectory}(\pi)$ for each $\ell \in [L]$.
    \State Add $\tau_\ell^k$ to replay buffer $\mathcal{D}_\ell^k \gets \mathcal{D}_\ell^{k-1} \cup \tau_\ell^k$.
    \State Let $\mathcal{D}^k = \cup^L_{\ell=1} \mathcal{D}^k_\ell$
    \State $ c_{w^{k}} \gets \textsc{UpdateCost}(c_{w^{k-1}}, \mathcal{D}_{\expert}, \mathcal{D}^k, \alpha)$ using the Base Method (such as CSIL, $f$-IRL or ML-IRL ). 
    \For{$\ell = 1$ to $L$}
        \State $Q_{\theta^{k+1}_\ell}, Q_{\theta_\ell^{k+1,\text{targ}}} = \textsc{UpdateCritics}(\mathcal{D}^k_\ell, \pi_{\psi^k}, \eta, \gamma, c_{\theta^k})$
    \EndFor
    \State $\bc{Q^{k+1}(s,a)}_{s,a \in \mathcal{D}^k} = \textsc{OptimisticQ-NN}(\mathcal{D}^k, 
\bc{Q_{\theta^{k+1}_\ell}}^L_{\ell=1}, \sigma)~~~~~~$ (see \Cref{alg:optQnn})
    \State Define the loss
    $\mathcal{L}^k_\pi = \frac{1}{\abs{\mathcal{D}^k}} \sum_{s,a \in \mathcal{D}^k}\left( -\eta \log \pi_{\psi^k}(a|s) +  Q^{k+1}(s,a) \right).
$
    \State Update policy weights to $\psi^{k+1}$ using Adam \cite{Kingma:2015} on the loss $\mathcal{L}^k_\pi$.
\EndFor
\State \textbf{Return} $\pi$
\end{algorithmic}
\end{algorithm}

\begin{algorithm}[H]
\caption{\textsc{UpdateCritics} \label{alg:updateCritics}}
\begin{algorithmic}[1]
\REQUIRE $\mathcal{D}^k_\ell, \pi_{\psi^k}, \alpha, \gamma,  c_w$
    \State Let $\mathcal{B} = \bc{s_i, a_i, r, s'_i, \text{done}_i}^N_{i=1}$ be a minibatch sampled from $\mathcal{D}$
    \State $a'_i \gets \pi(s'_i)$ for all $i \in [N]$.
    \State Define $Q_{\theta_{\ell}^{k,\text{targ}}}(s_i,a_i) \gets \max \left( Q_{\theta_{\ell}^{k,\text{targ},(1)}}(s_i,a_i), Q_{\theta_{\ell}^{k,\text{targ},(2)}}(s_i,a_i) \right)$ for all $s_i,a_i \in \mathcal{B}$. 
    \State $\text{Backup}_i \gets c_w(s_i,a_i) + \gamma (1 - \text{done}_i) \left( Q_{\theta_{\ell}^{k,\text{targ}}}(s_i,a_i) + \alpha \log \pi_{\psi^k}(a'_i|s'_i) \right)$
    \State $\mathcal{L}_{\theta^{k,(1)}_\ell} = \frac{1}{N} \sum_{i=1}^N \left( Q_{\theta^{k,(1)}_\ell}(s_i, a_i) - \text{Backup}_i \right)^2$
    \State $\mathcal{L}_{\theta^{k,(2)}_\ell} = \frac{1}{N} \sum_{i=1}^N \left( Q_{\theta^{k,(2)}_\ell}(s_i, a_i) - \text{Backup}_i \right)^2$
    \State $\theta^{k+1,(1)}_\ell \gets \theta^{k,(1)}_\ell - \eta_Q \nabla \mathcal{L}_{\theta^{k,(1)}_\ell}$
    \State $\theta^{k+1,(2)}_\ell \gets \theta^{k,(2)}_\ell - \eta_Q \nabla \mathcal{L}_{\theta^{k,(2)}_\ell}$.
    \State $\theta^{k+1, \text{targ},(1)} \gets (1 - \tau_{\text{targ}})\theta^{k, \text{targ},(1)} + \tau_{\text{targ}} \theta^{k,(1)} $.
    \State $\theta^{k+1, \text{targ},(2)} \gets (1 - \tau_{\text{targ}})\theta^{k, \text{targ},(2)} + \tau_{\text{targ}} \theta^{k,(2)} $.
    \State $Q_{\theta^{k+1}_\ell}(s,a) = \max \br{ Q_{\theta^{k+1,(1)}_\ell}(s,a), Q_{\theta^{k+1,(2)}_\ell}(s,a)}$ for all $s,a \in \mathcal{D}^k$.
    \State $Q_{\theta^{k+1, \text{targ}}_\ell}(s,a) = \max \br{ Q_{\theta^{k+1, \text{targ},(1)}_\ell}(s,a), Q_{\theta^{k+1, \text{targ},(2)}_\ell}(s,a)}$ for all $s,a \in \mathcal{D}^k$.
    \State \textbf{return} $Q_{\theta^{k+1}_\ell}, Q_{\theta^{k+1, \text{targ}}_\ell} $.
\end{algorithmic}
\end{algorithm}
\subsection{Instantiating the cost update}
We show after how the algorithmic template in \Cref{alg:meta} captures different imitation learning algorithms just changing the cost update.
For example, $f$-IRL with the reversed KL divergence (RKL) can be seen as \Cref{alg:meta} with the cost update described in \Cref{alg:fIRL}.
Moreover, our SOAR+RKL is obtained plugging in the cost update in \Cref{alg:fIRL} in \Cref{alg:deepmeta}.
\begin{algorithm}[H]
\caption{\textsc{UpdateCost} for RKL ($f$-IRL for reversed KL divergence) \cite{ni2021f} \label{alg:fIRL}}
\begin{algorithmic}[1]
\REQUIRE $c, \mathcal{D}_{\expert}, \mathcal{D}_{\pi^k}, \alpha$, divergence generating function $f(x) = - \log (x)$ for the reversed KL divergence, prior distribution over trajectories $p(\tau)$.
    \State $\rho_w(\tau) = \frac{1}{Z} p(\tau) e^{-c_w(\tau)}$
    \State $\chi^\star \gets \argmax_\omega \mathbb{E}_{s \sim \mathcal{D}_{\expert}} [\log D_\chi(s)] + \mathbb{E}_{s \sim \mathcal{D}^k} [\log (1 - D_\chi(s))]$
    \vspace{0.3em}
    \State Estimate the density ratio:
    \State $\frac{\rho_E(s)}{\rho_w(s)} = \frac{D_{\chi^\star}(s)}{1 - D_{\chi^\star}(s)}$
    \vspace{0.3em}
    \State Compute the stochastic gradient \begin{align*}\widehat{\nabla_w} &= \frac{1}{ T} \mathbb{E}_{\tau \sim \rho_w} 
    \bs{
    \sum_{t=1}^T h_f \left( \frac{\rho_E(s_t)}{\rho_w(s_t)}  \right) \cdot
    \left( - \sum_{t=1}^T \nabla_\theta c_w(s_t) \right)
    } \\&\phantom{=}- \frac{1}{ T} \mathbb{E}_{\tau \sim \rho_w} 
    \bs{
    \sum_{t=1}^T h_f \left( \frac{\rho_E(s_t)}{\rho_\theta(s_t)}  \right)} \cdot
    \mathbb{E}_{\tau \sim \rho_\theta} 
    \bs{\left( - \sum_{t=1}^T \nabla_\theta c_w(s_t) \right)
    }\end{align*}
    \State $w \gets w - \alpha \widehat{\nabla_w}$
    \State \textbf{Return} $c_w$
\end{algorithmic}
\end{algorithm}
Next, we present the cost update for the algorithm ML-IRL \cite{zeng2022maximum}. We present it for the state-action version. The state-only version is obtained simply omitting the action dependence everywhere.
\begin{algorithm}[H]
\caption{\textsc{UpdateCost} for ML-IRL (State-Action version) \cite{zeng2022maximum} \label{alg:MLIRL}}
\begin{algorithmic}[1]
\REQUIRE $c_{w}, \mathcal{D}_{\expert}, \tau^k = \bc{s^k_t,a^k_t}^{L^k}_{t=1}, \alpha$.
    \State Sample a state-action trajectory $\tau_E = \bc{s_t^E, a_t^E}^{L_E}_{t=1}$ from the expert dataset $\mathcal{D}_\expert$ where $L_E$ is a geometric random variable with parameter $(1-\gamma)^{-1}$.
    \State Compute the stochastic loss 
    $$ \widehat{\mathcal{L}}_{w}  =\sum^{L_E}_{t=0} \gamma^t c_w(s^E_t,a^E_t) -  \sum^{L^k}_{t=0} \gamma^t c_w(s^k_t,a^k_t)
    $$

    \State $w \gets w - \alpha \nabla_w \widehat{ \mathcal{L}_w}$
    \State \textbf{Return} $c_w$
\end{algorithmic}
\end{algorithm}

To conclude, we present the cost update for CSIL. Notice that since the cost used by CSIL does not leverage the information of the policy at iteration $k$ we can move the cost update before the main loop and keep a constant cost function fixed during the training of the policy.
Notice that the CSIL the reward is simply given by the log probabilities learned by the behavioural cloning policy. Therefore the reward parameters in CSIL coincides with the parameters of the behavioral cloning policy network.
We point out that using a reward of this form is similarly done in \cite{vieillard2020munchausen}. We plan to explore further the connection between Munchausen RL and CSIL in future work.

\begin{algorithm}[H]
\caption{\textsc{UpdateCost} for CSIL \cite{watson2023coherent}\label{alg:CSIL}}
\begin{algorithmic}[1]
\REQUIRE $\mathcal{D}_{\expert}$.
    \State Compute the behavioral cloning policy finding an approximate solution to the following problem.
    $$
    w^\star = \argmax_{w} \sum_{s,a \in \mathcal{D}_\expert}\log \pi_{w}(a|s)
    $$
    \State \textbf{Return} $c_{w^\star}(s,a) = - \log \pi_{w^\star}(a|s)$
\end{algorithmic}
\end{algorithm}


\end{document}